\documentclass[dvipsnames]{article}

\usepackage[preprint]{icml2026}

\usepackage[utf8]{inputenc} 
\usepackage[T1]{fontenc}    
\usepackage{hyperref}       
\usepackage{url}            
\usepackage{booktabs}       
\usepackage{amsfonts}       
\usepackage{nicefrac}       
\usepackage{microtype}      
\usepackage{xcolor}         
\usepackage{amsmath}
\usepackage{amsthm}
\usepackage{multirow}
\usepackage{adjustbox}
\usepackage{bbm}
\usepackage{graphicx}
\usepackage{caption}

\newtheorem{lemma}{Lemma}
\newtheorem{proposition}{Proposition}
\newtheorem{definition}{Definition}
\newtheorem{remark}{Remark}

\newcommand{\dotvect}[2]{#1^{\top} #2}
\newcommand{\dotmat}[3]{#1^{\top} #3 #2}
\newcommand{\dotangle}[2]{\langle #1, #2 \rangle}
\newcommand{\spec}{\text{spec}\,}
\newcommand{\bspec}{\text{spec}_z}
\newcommand{\sheaflap}{\Delta^{\mathcal{F}}}

\newcommand{\mtsheaflap}{\tilde{\Delta}^{\mathcal{F}(t)}}
\newcommand{\tsheaflap}{\Delta^{\mathcal{F}(t)}}
\newcommand{\first}[1]{\textbf{\textcolor{BrickRed}{#1}}}
\newcommand{\second}[1]{\textbf{\textcolor{NavyBlue}{#1}}}
\newcommand{\third}[1]{\textbf{\textcolor{ForestGreen}{#1}}}

\icmltitlerunning{Geometrical fairness in graph neural networks}

\begin{document}

\twocolumn[
\icmltitle{Geometrical fairness in graph neural networks}

  \begin{icmlauthorlist}
    \icmlauthor{Arturo Pérez-Peralta}{uc3m}
    \icmlauthor{Sandra Benítez-Peña}{uc3m,ibdat}
    \icmlauthor{Blas Kolic}{ibdat}
    \icmlauthor{Rosa E. Lillo}{uc3m,ibdat}
  \end{icmlauthorlist}

  \icmlaffiliation{uc3m}{Department of Statistics, University Carlos III of Madrid, Spain}
  \icmlaffiliation{ibdat}{uc3m-Santander Big Data Institute}

  \icmlcorrespondingauthor{Arturo Pérez-Peralta}{100507525@alumnos.uc3m.es}

  \icmlkeywords{Machine Learning, Fairness}
  \vskip 0.3in
]

\printAffiliationsAndNotice{}

\begin{abstract}
Graph-based learning methods have become increasingly prominent due to their strong performance across diverse applications. Among these, recent frameworks grounded in diffusion processes provide a unifying perspective that extends traditional graph neural network formulations while addressing limitations of standard message-passing mechanisms. Despite these advances, concerns remain regarding the fairness of such models, as they may propagate or amplify biases present in the data.
In this work, we introduce a fairness-aware adaptation of graph-based diffusion by modifying the underlying Laplacian operator. Our approach incorporates multiple complementary transformations, including subspace projections, spectral adjustments, and frequency-based filtering, to mitigate bias-related components. Leveraging the intrinsic smoothing properties of graph diffusion, we provide a principled analysis of the resulting behavior and establish theoretical insights into fairness properties.
We evaluate the proposed framework on both synthetic and real-world datasets, demonstrating that it achieves competitive performance while improving fairness metrics with limited additional computational cost.
\end{abstract}

\begin{figure}[t]
    \centering
    \includegraphics[width=1.0\linewidth]{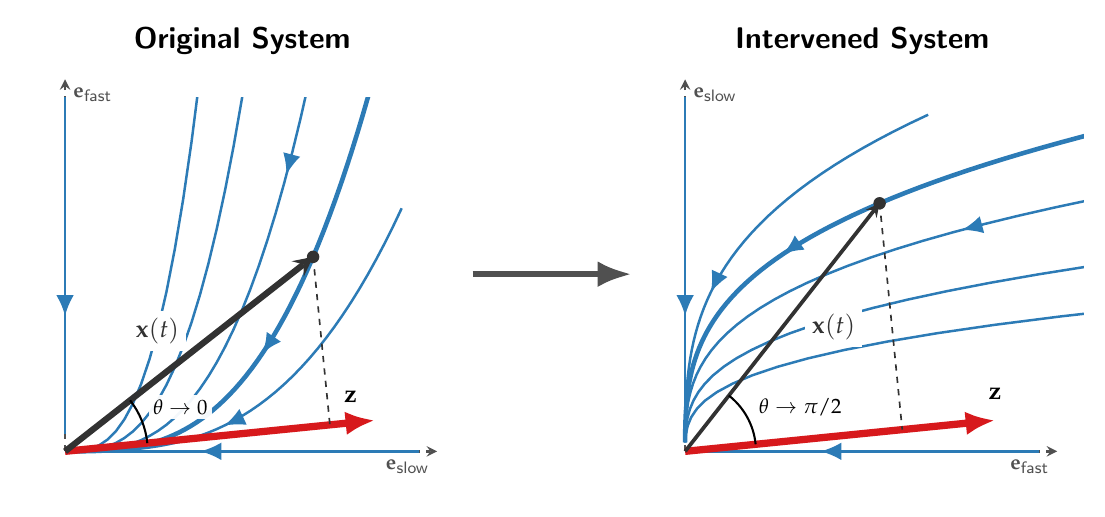}
    \caption{Slow eigenvectors closely aligned with the sensitive attribute on the Laplacian encourage unfair solutions. Geometric fairness methods rescale the slow biased eigenvalues, thus increasing their frequency and encouraging energy dissipation in the direction aligned with the sensitive attribute.}
    \label{fig:Intro1}
\end{figure}

\section{Introduction}

Graph Neural Networks (GNNs) have become increasingly widespread in recent years, driven by their success across fields as diverse as content recommendation, particle physics, social sciences, and quantum chemistry \cite{QuantumChemistry, Recommendation, ParticlePhysics, Social}. Their ability to leverage topological information arising from the message-passing paradigm has led to their extended adoption in tasks ranging from fraud detection to drug design \cite{FraudDetection, DrugDesign}.\\
However, despite their merits, GNN architectures are not without their shortcomings, including their tendency to equalize the hidden representations of different individuals as more layers are stacked, also known as oversmothing, \cite{Oversmoothing1, Oversmoothing2}, their drop in utility when graphs relate dissimilar individuals, also known as heterophily \cite{heterophily2, heterophily1}, and their tendency to exacerbate harmful demographic biases \cite{Start1, Start2, Start3}. All three of these points have been addressed in the literature, with Neural Sheaf Diffusion (NSD) \cite{nsd} providing a principled framework to handle oversmoothing and heterophily by generalizing the graph heat equation to include matrix weights that enhance its expressiveness. However, this framework lacks a treatment of its properties in relation to fairness, limiting its deployment in critical decision-making contexts such as finance or healthcare~\cite{Healthcare, Finance}.\\
While the literature of \emph{Fair Machine Learning} offers an extensive catalog of bias mitigation techniques which integrate demographic information into the Machine Learning pipeline, such as careful data curation, model intervention, and adversarial architectures \cite{Reweighting, MetaFair, FairGNN}, its interaction with phenomena including oversmoothing and heterophily is less understood~\cite{heterophily2}. This is why a fairness-aware NSD can benefit from its properties and generality, enabling a model that addresses the three phenomena at once with a unified perspective. To perform this task, we first acknowledge that a naive implementation of NSD can lead to unfair outcomes, as illustrated in Figure~\ref{fig:Intro1}: when the slower eigenvalues of the heat diffusion process are aligned with unfair directions, the resulting system converges to a relatively unfair solution. This is why a geometric method addressing the bias of the underlying diffusion process is required, and provides the basis for our approach: increasing the frequency of slower, unfair eigenvectors encourages energy dissipation in these problematic directions, leading to a fairer result.
\paragraph{Contributions.} This paper introduces algorithmic fairness mechanisms into graph models by combining geometrical and spectral arguments within the Neural Sheaf Diffusion framework, thus creating a fairness-aware graph model. Our ideas are motivated by the fact that the graph heat equation acts as a low-pass filter. Therefore, increasing the frequency of slow-biased eigenvectors improves fairness by reducing the dependence on these biased components. This observation serves as the basis for establishing the theoretical foundations of the methods, which are thoroughly studied in the language of time-dependent dynamical systems, culminating in three in-processing methods that modify the Laplacian to achieve fairness. Concretely, the three modifications are given by a projection onto a biased vector, the projection onto a subspace of eigenvalues of the graph Laplacian correlated with the biased vector, and a spectral polynomial filter. This results in three different methods with the following desirable properties:
\begin{itemize}
    \item \textbf{Soundness:} The methods are the result of theoretical inquiries in time-varying systems, relying on results which can be used to build further fairness processors.
    \item \textbf{Minimal overhead:} The bulk of their cost can be relegated to off-line computations, while their actual implementation adds negligible overhead.
    \item \textbf{Flexibility:} The models account for the utility-fairness trade-off through a parameter with a clear effect on fairness.
    \item \textbf{Competitiveness:} The proposed methods achieve notable results when compared to the state of the art in the literature of Fair Machine Learning.
\end{itemize}
We have chosen to focus on cellular sheaves~\cite{Opinion, nsd} due to the general framework they introduce, which encompasses common graph architectures such as Graph Convolutional Networks and Graph Attention Networks.

\paragraph{Structure.} Section~\ref{StateOfTheArt} discusses the state of the art of fair classification in the context of Machine Learning over graphs, while Section~\ref{Theory} is the exposition of the theoretical background, delving into the details of Fair Machine Learning and graph models based on cellular sheaves. Section~\ref{GeometricFairness} lays the theoretical foundation of this paper, defining a geometric notion of fairness and introducing the aforementioned debiasing methods based on projections and spectral filters, and providing a thorough study of the proposed processors in terms of their fairness properties and the deviations incurred by their introduction. Next, Section~\ref{Experiments} establishes the empirical properties of the different models through a series of experiments, including a synthetic dataset to test the models in a controlled environment, an exploration of the effect of the methods on the spectrum of the diffusion operator, and a comparison with a myriad of other models present in the literature, both tabular and graph-based. Finally, Section~\ref{Conclusion} summarizes the main results and offers insights on how this work could be further extended.

\section{Related work}
\label{StateOfTheArt}

\paragraph{Fairness.} Statistical learning can exacerbate demographic biases \cite{Justice}, hindering deployment in critical sectors like finance or healthcare \cite{Healthcare, Finance} and driving demand for fair models \cite{Academic, industry}. While various metrics exist quantifying bias \cite{ManyDefinitions}, most are equivalent to statistical parity, equal odds, or calibration \cite{FairMLBook}; we focus on the former. Mitigation algorithms are categorized by their pipeline stage: pre-processors (e.g., reweighting \citep{Reweighting}, disparate impact removal \citep{DisparateImpactRemover}) modify training data; in-processors (e.g., adversarial debiasing \citep{AdversarialDebiasing}, prejudice index regularization \citep{PrejudiceIndexRegularizer}) adjust the training procedure; and post-processors (e.g., equal odds \citep{EqualOdds}, option rejection \citep{RejectOption}) alter final predictions.
Previous works in fairness leveraging geometrical ideas to measure and mitigate bias include GEOFFAIR, a geometrical framework which identifies models, distributions and hypothesis as sets and vectors thus gaining valuable insights arising from the more intuitive nature of geometry~\cite{Geoffair}, and fair interpretable representations obtained through a pre-processing step based on the geometrical projection of the input data into a bias-free subspace~\cite{Geometric2}. However, these notions are grounded in static geometry. This is insufficient for models like Graph Attention Networks and Neural Sheaf Diffusion, which can be linked to a time-dependent heat equation, thus requiring methods capable of handling dynamical systems. Our work sets itself apart by introducing three such in-processors.

\paragraph{Graph models.} Graph models have gained popularity in recent years due to their success in tasks as diverse as drug design, particle physics, and content recommendation \cite{Recommendation, ParticlePhysics, DrugDesign}. The underlying Graph Neural Networks (GNNs) are based on the message-passing paradigm, which relies on an aggregation scheme that incorporates neighbor features, thus enriching the representation of individuals and improving performance on tasks that benefit from these notions~\cite{QuantumChemistry, FraudDetection}. On a more technical level, the message-passing mechanism allows the underlying model to learn topological information closely related to graph diffusion~\cite{Diffusion} and the Weisfeiler-Lehman graph isomorphism test~\cite{GIN}. Under these premises, many different message-passing schemes have been devised, including Graph Convolutional Networks \cite{GCN}, Graph Attention Networks \cite{GAT}, and GraphSAGE \cite{SAGE}, among others. Despite these successes, the performance of these models is highly dependent on the underlying graph, with limitations arising due to topological bottlenecks, also known as oversquashing \cite{Oversquashing2, Oversquashing1}, the equalization of the hidden representations of the models as many layers are stacked, also known as oversmoothing \cite{Oversmoothing1, Oversmoothing2}, and the loss of accuracy when dealing with heterophilic graphs, as the assumption of similarity between connected neighbors is broken~\cite{heterophily2, heterophily1}. Neural Sheaf Diffusion (NSD) \cite{nsd} is a recent model that provides a principled explanation of oversmoothing and heterophily by using matrix weights instead of scalar weights, generalizing the graph heat equation and offering the expressivity needed to better handle common shortcomings. These methods have received multiple extensions in recent years, with the most relevant for this paper being the development of learnable spectral filters using Chebyshev polynomials \cite{polynsd}.
\paragraph{Fairness in graphs.} Graph fairness faces unique challenges due to the interplay between topology and demographic attributes, which exacerbates bias when neighborhoods are influenced by sensitive features \cite{Start1, Start2, Survey1, Start3, Survey2}. Topologies can amplify \cite{FairEDIT, Start1} or mitigate \cite{FairSpecRewire} bias, yet reconciling these effects remains an open problem. Fairness processors over graphs include the rewiring pre-processing technique, which has been shown to mitigate bias via fairness gradients, such as FairEDIT \cite{FairEDIT}, edge dropout, such as FairDROP \cite{FairDROP}, or spectral methods \cite{FairSpecRewire}. On another note, some successful in-processing approaches include counterfactual methods, such as NIFTY \cite{Nifty}, and adversarial debiasing, as seen in models such as FairGNN \cite{FairGNN} and FairSIN \cite{FairSIN}. Despite recent work pointing to the relationship between heterophily and fairness~\cite{FairnessHeterophily2, FairnessHeterophily1} and the advances in principled frameworks for handling complex topological problems such as the aforementioned heterophily and oversmoothing, like NSD~\cite{nsd}, there is no overarching perspective that connects all of these issues. This is why we aim to expand the current landscape by introducing fairness-encouraging modifications to NSD, motivated by the geometrical and spectral properties of the heat kernel. The philosophy of the approach is illustrated in Figure~\ref{fig:Intro1}: a naive implementation of NSD, which can be understood as a dynamical system, can produce an unfair result if the underlying dynamics are aligned with biased directions, thus requiring an intervention. Finally, while previous work already highlighted the relationship between cellular sheaves, spectral filters, and fairness \cite{FairSpecFilter, FairSheafDiffusion}, the resulting models amount to standard Graph Convolutional Networks lacking the generality of NSD.

\section{Background}
\label{Theory}

\paragraph{Fair classification.} The central problem addressed in this paper is fair node classification. The input consists of a feature matrix $X \in \mathbb{R}^{n \times q}$, an undirected graph $G = (V,E)$ with $n$ nodes, binary labels $Y \in \{0,1\}^n$, and a binary sensitive attribute $Z \in \{0,1\}^n$. The goal is to learn a model $f_\theta$ that predicts $Y$ from $(X,G)$ while mitigating bias with respect to $Z$ under a given fairness metric. On the one hand, the actual model outputs, $\hat{p}_i = f_{\theta}(x_i; G)$, represent the probability that observation $i$ belongs to the positive label class, $y_i=1$, and, on the other hand, the final predictions are made by thresholding this probability using the Bayes-optimal rule, $\hat{y}_i = \mathbbm{1}(\hat{p}_i \geq 0.5)$, where $\mathbbm{1}$ denotes the indicator function. Model training is done through the gradient descent paradigm using the binary cross-entropy loss function, $\mathcal{L}(f_{\theta};\theta) = \sum_{i=1}^n [y_i \log \hat{p}_i + (1-y_i) \log (1-\hat{p}_i)]$. The quality of the estimation will be measured by the accuracy of the predictions, $Acc. = \hat{\mathbb{P}}(\hat{Y} = Y)$ where $\hat{\mathbb{P}}(\cdot)$ denotes the empirical probability of the event. On the other hand, the prediction bias will be measured through the statistical parity metric, $\Delta_{SP} = |\hat{\mathbb{P}}(\hat{Y}=1 | Z = 0) - \hat{\mathbb{P}}(\hat{Y}=1 | Z = 1) |$. Finally, to scalarize the multiobjective optimization problem at hand and quantify the utility-fairness trade-off, we use the Euclidean distance to the ideal point $Acc. = 1$, $\Delta_{SP}=0$, expressed as $DTI = \sqrt{(1-Acc.)^2 + \Delta_{SP}^2}$.

\paragraph{Cellular sheaves and Neural Sheaf Diffusion.} The main goal of this paper is to debias graph models while minimizing performance loss, and the insights presented in the next section can be easily generalized to multiple models. However, as previously stated, we focus on Sheaf Diffusion, which we present below.
\begin{definition}
A cellular sheaf $(\mathcal{F}, G)$ on an undirected graph $G = (V,E)$ consists of a vector space (the stalk) $\mathcal{F}(v)$ for every node $v\in V$ and $\mathcal{F}(e)$ for every edge $e \in E$, as well as linear maps (the restrictions) $\mathcal{F}_{v \trianglelefteq e} : \mathcal{F}(v) \rightarrow \mathcal{F}(e) $ for every incident node-edge pair $v \trianglelefteq e$. 
\end{definition}
The space of $0-$ and $1-$ cochains are $\mathcal{C}^0(\mathcal{F}, G) = \bigoplus_{v\in V} \mathcal{F}(v)$ and $\mathcal{C}^1(\mathcal{F}, G) = \bigoplus_{e\in E} \mathcal{F}(e)$. Given an arbitrary orientation for all edges, we may define the coboundary operator $\delta: \mathcal{C}^0(\mathcal{F}, G) \rightarrow \mathcal{C}^1(\mathcal{F}, G)$ which acts on the edges as $(\delta x )_e = \mathcal{F}_{u \trianglelefteq e} x_u - \mathcal{F}_{v \trianglelefteq e} x_v $ for all $0-$cochains $x$. This, in turn, can be used to define the sheaf Laplacian, given by $L^{\mathcal{F}} = \delta^{\top} \delta$, which is a symmetric semidefinite positive block matrix with diagonal blocks $L^{\mathcal{F}}_{vv} = \sum_{e = (u,v)}  \mathcal{F}^{\top}_{v \trianglelefteq e} \mathcal{F}_{u \trianglelefteq e}$ and off-diagonal blocks $L^{\mathcal{F}}_{uv} = -\mathcal{F}^{\top}_{v \trianglelefteq e} \mathcal{F}_{u \trianglelefteq e}$. By normalizing the stalks, we arrive at the normalized Laplacian, $\sheaflap = D^{-1/2} L^{\mathcal{F}} D^{-1/2}$, where $D$ is a block diagonal matrix whose diagonal elements are $L^{\mathcal{F}}_{vv}$. Therefore, it is symmetric, semidefinite positive, and its square root is well defined. If this matrix were not of full rank, the negative power is implied to mean the Moore-Penrose pseudoinverse. Finally, the bilinear form defined by the normalized sheaf Laplacian receives the name of Dirichlet energy, $E^{\mathcal{F}}(x) = \dotmat{x}{x}{\sheaflap} $\\
We now have the machinery required to define Sheaf Diffusion, given by the following ordinary differential equation:
\begin{equation}
\begin{cases}
    \dot{x}_t = - \tsheaflap x_t \\ x_0 \in \mathcal{C}^0 (\mathcal{F}(0), G)
\end{cases}
\label{SheafDif}
\end{equation}
Note that we let the underlying sheaf $\mathcal{F}(t)$ vary with time. For simplicity, assume that the stalk dimension stays constant and only the restriction maps change. While the properties of the time-constant Sheaf Diffusion process are well understood \cite{Opinion}, the time-dependent dynamics are more problematic and require a more careful treatment. In \cite{nsd}, they study the time limit properties of such systems and leverage their insights in conjunction with the Euler discretization scheme to create a deep learning model in Neural Sheaf Diffusion (NSD):
\begin{equation}
    X^{t+1} = X^t - \sigma(\tsheaflap (I_{n} \otimes W_1^t) X^t W_2^t),
    \label{NSD}
\end{equation}
where $\otimes$ denotes the Kronecker product, $W_1^t \in \mathbb{R}^{d\times d}, W_2^t \in \mathbb{R}^{c_1 \times c_2}$ are learnable matrices, with $c_1, c_2$ being the number of input and output feature channels, respectively. Finally, the restriction maps are learnt, usually using a multi-layer perceptron on the concatenation of the vector of features of node $u$ and $v$, $\tilde{\mathcal{F}}^{(t)}_{u \leq e} = \Phi^{(t)}(x_u, x_v) = MLP^{t}(x_u \| x_v)$, where $\|$ denotes the concatenation of both vectors. In any case, there are three parametrizations for the empirical restriction matrices: diagonal, whose off-diagonal elements are set to zero; orthogonal, implementing an orthogonal transformation of the Euclidean space; and general, learning an arbitrary linear transformation, increasing both generality and numerical instability.

Before studying the properties of \eqref{SheafDif} in relation to fairness, we remark on the generality of the NSD model. As noted in the original NSD paper, \cite{nsd}, using the trivial sheaf recovers the original Graph Convolutional Network. This, in combination with adding a learnable parameter $(1+\epsilon)$, returns the Graph Isomorphism Network. In any case, setting $d=1$ results in a Graph Attention network. This generality is where the sheaf formalism pays off, immediately extending the forthcoming discussion to all of these models. In Appendix~\ref{app:Results}, we delve into the details and empirically motivate the use of sheaf models to obtain improved outcomes in both fairness and accuracy.

\section{Geometric Fairness and Fair Neural Sheaf Diffusion}
\label{GeometricFairness}
In this Section, we establish three different procedures to debias the Neural Sheaf Diffusion (NSD) model and explore its theoretical properties. We assume fairness metrics can be encoded as the dot product of a vector $x$ representing model outputs and a given unit vector $z$ representing membership into a sensitive demographic group.
\begin{definition}[Geometric Fairness Metric] Given a unit vector $z\in \mathbb{R}^m$, the geometric fairness metric induced by $z$ is given by $F_z(x) = |\dotangle{z}{x}|$. 
\end{definition}

This definition deserves some discussion. First, using a vector $z\in \mathbb{R}^n$, we can encode a continuous proxy to the fairness metrics by applying this definition to the probabilities $\hat{p}_i$, which is common in the literature \cite{PROXY1, PROXY2}. Two prominent fairness metrics which can be encoded in this manner are Statistical Parity ($\Delta_{SP}$), using the normalization of $\hat{z} = \mathbbm{1}(Z=1)/|Z=1| - \mathbbm{1}(Z=0)/|Z=0|$, and the individual-fairness-oriented kNN-Consistency ($\text{Con}_k$) \cite{CONSISTENCY}, using the normalization of $\hat{z} = \mathbf{1} - \frac{1}{k}\mathbf{1}^{\top} A_k$ with $A_k$ the adjacency matrix of the k-nearest neighbors (kNN) graph. In both cases, a simple calculation shows that the fairness metric is given by $|\dotangle{\hat{z}}{p}|$. Theoretically, Equal Odds can also be expressed in this manner, although it requires knowledge of the class labels. In any case, the results presented below can be applied to a myriad of fairness metrics because our framework is sufficiently general. 

We evaluate the Geometric Fairness Metric on solutions of the System \eqref{SheafDif}. However, simple continuity arguments show that using a linear layer to transition from the original diffusion space (say, $\mathbb{R}^{n\times q}$) to the prediction space ($\mathbb{R}^n$) allows us to establish a bound involving the Geometric Fairness Metric of the predicted probabilities and the diffusion space~\footnote{Relating the Geometric Fairness Metric of the probabilities and the predictions is more difficult when not outright impossible due to the discontinuity introduced in the threshold function. As previously mentioned, we will work with the probabilities and treat them as a continuous proxy of the actual predictions.}. This relationship motivates the following definition.
\begin{definition}[Geometric Fairness]
    The system given by Eq. \eqref{SheafDif} is fair with respect to the Geometric Fairness Metric $F_z$ if $F_z(t) = |\dotangle{z}{x_t}| \xrightarrow[t\rightarrow \infty]{} 0$. Equivalently, we say that the system is geometrically fair with respect to $z$.
\end{definition}
The original system \eqref{SheafDif} is not guaranteed to achieve Geometrical Fairness. For example, assuming a constant sheaf Laplacian whose kernel contains the sensitive vector $z$, it is clear that in the time limit, the orbit of the system is fully aligned with $z$, thus making Geometric Fairness impossible. To address this problem, we propose three modifications to the original Sheaf Laplacian, $\mtsheaflap$, such that the resulting system:
\begin{equation}
\begin{cases}
    \dot{x}(t) = - \mtsheaflap x(t) \\ x(0) = x_0 \in \mathcal{C}^0 (\mathcal{F}(0), G)
\end{cases},
\label{SheafModified}
\end{equation}
is guaranteed to achieve Geometric Fairness:
\paragraph{Fair Vector Projection (FVP).} The first debiasing method relies on a rank-one perturbation to the original system resulting from the addition of a projection onto the biased subspace:
\begin{equation}
\tsheaflap  \longrightarrow \mtsheaflap = \tsheaflap + \gamma z z^{\top}, \quad \gamma > 0.
\label{VectorProjection}
\end{equation}
This projection pushes the vector $z$ onto the spectrum of $L(t)$, in order to dissipate energy in the biased subspace.
FVP can be efficiently implemented without using the dense matrix $z z^{\top}$ by instead relying on vector dot product whenever possible:
\begin{equation*}
    \mtsheaflap x =  \underbrace{\tsheaflap
x}_{\text{sparse}} + \gamma \underbrace{\dotangle{z}{x}}_{\text{dot product}} z.
\end{equation*}
Therefore, FVP adds minimal computational overhead of $O(n_d)$ operations per layer of diffusion, where $n_d$ is the dimension of $\mathcal{C}^0(G, \mathcal{F}(t))$. 

The remaining two methods instead target slow-dissipation eigenvectors with high correlation with the sensitive vector $z$, and require the concept of the biased spectrum:
\begin{definition}
    Let $L$ be a linear operator, $z$ a unit vector. The \emph{biased spectrum} of an operator $L$ with respect to $z$, denoted by $\bspec L$, is composed of all the eigenpairs $(\lambda_i, v_i) \in \spec L$ with non-zero projection onto the space  $\langle z, v_i\rangle \neq 0$.
\end{definition}
In an abuse of notation, we use the term spectrum indistinctly to refer to the eigenpairs, eigenvalues, or eigenvectors, making the distinction clear from context. The idea now is that the linear system \eqref{SheafDif} acts as a low-pass filter, mitigating the biggest eigenvalues. Therefore, the high-frequency components of the biased spectrum quickly converge to zero and do not pose a problem, while the smaller eigenvalues persist for a long time, preventing the mitigation of the Geometric Fairness Metric. This insight leads to the following two methods, which target these low frequency biased spectrum:

\paragraph{Fair Spectral Projection (FSP).} Consider a subset of the biased spectrum $J_t \subset \bspec L(t)$, in practice those of lowest frequency. We can target these eigenvalues by considering their projection matrices:
\begin{equation}
    \tsheaflap \longrightarrow \mtsheaflap = \tsheaflap + \sum_{j \in J_t} \gamma_{t, j} v_j(t) v_j(t)^{\top},
\label{SpectralProjection}
\end{equation}
where $\gamma_{t, j} > 0$ for all $t> 0$ and $j\in J_t$, although in practice we will choose the same value of $\gamma$ for all $t$ and $j$, $\gamma_{t,j} = \gamma$. Applying a similar reasoning as before, the matrix product $\mtsheaflap x$ adds an overhead of $O(d_n)$. The more costly part of this new procedure lies instead in the computation of the biased spectrum, which can nonetheless be performed just once in an offline fashion. In any case, the computational cost of approximating the first $k$ eigenvalues of a sparse matrix with $n_{nz}$ non-zero elements in $n_s$ iterations is $O(k n_s n_{nz})$ \cite{BigOEigen}, a cost that must be incurred only once. 

\paragraph{Fair Spectral Filter (FSF).} The last method relies on the same idea as FSP, but, instead of projections, it uses the concept of spectral filters. Given the eigendecomposition of a diagonalizable matrix, $L = V \Lambda V^{\top}$, a spectral filter on $L$ is a matrix of the form $F = V \Phi(\Lambda) V^{\top}$ with $\Phi$ a function acting on the eigenvalues. In our case, given a subset of biased eigenvalues $J_t \subset \bspec \tsheaflap$, the goal is to increase their frequency while leaving the rest of the spectrum intact. Therefore, an ideal filter would resemble the identity plus a sum of Gaussian peaks: 
\begin{equation*}
    \Phi_t(\lambda) = \lambda + \sum_{j\in J_t} \gamma_j \exp{\left(-\frac{(\lambda - \lambda_j)^2}{2\sigma_j^2} \right)}  .
\end{equation*}
In practice, we use the same values for all peak heights and localization parameters, $\gamma = \gamma_j$ and $\sigma = \sigma_j$, for all $j$ in $J_t$ and $t$. However, instead of using this ideal function, we rely on a polynomial filter leveraging the fact that a polynomial on $L$ satisfies $P(L) = V P(\Lambda) V^{\top}$. Thus, we perform a polynomial approximation of degree $K$ using Chebyshev polynomials, $P_t \approx \Phi_t$, which requires rescaling of the spectrum of the sheaf Laplacian from its original range, $[0, \lambda_{max, t}]$, to the unit interval, $[-1, 1]$, by doing $\tsheaflap_{s} = \frac{1}{\lambda_{max, t}}\tsheaflap - I$. Finally, we perform the substitution:
\begin{equation}
    \tsheaflap \longrightarrow \mtsheaflap = P_t(\tsheaflap_s).
    \label{SpectralFilter}
\end{equation}
Despite the filter being of the form $P_t(\tsheaflap) = \sum_{i=0}^K \alpha_i (\tsheaflap)^i$, we can avoid computing and storing the increasingly dense powers of $\tsheaflap$ by using the recurrent vectors $x_{k,t} = \tsheaflap x_{k-1,t}$, relying exclusively on matrix-vector multiplication at a cost of $O(n_{nz})$ per iteration. Therefore, the method adds a computational overhead of $O(Kn_{nz})$. On the other hand, this processor adds an additional step in the approximation of the polynomial filter, but the bulk of the cost lies again in the computation of the biased spectrum at $O(k n_s n_{nz})$. For the details on the implementation of the approximation through Chebyshev polynomials, please refer to \cite{Cheby}.

We are now ready to establish the theoretical properties of the proposed methods, starting with Geometrical Fairness and convergence speed. All proofs of the results can be found in Appendix ~\ref{app:Theory}.

\begin{proposition} 
\label{Convergence}
Methods \eqref{VectorProjection} to \eqref{SpectralFilter} are geometrically fair. 
\end{proposition}

Furthermore, we can characterize the speed of convergence up to a ball of certain radius:

\begin{proposition}
    \label{speed}
    The Geometric Fairness Metric decays exponentially at a rate $\gamma$ for FVP and $\lambda_{min, z}$ for FSP and FSF towards a ball of radius $O(\lambda_{max} / \gamma)$ for FVP and $O(\lambda_{max} / \lambda_{min, z})$ for FSP and FSF, where $\lambda_{min, z}$ is the smallest eigenvalue of the biased spectrum of the modified Laplacian and $\lambda_{max}$ is the highest eigenvalue of its full spectrum.
\end{proposition}

Note the stark difference between Propositions~\ref{Convergence} and~\ref{speed}: Proposition~\ref{Convergence} establishes unconditional convergence to zero under the geometrical interventions thanks to Barbalat's lemma, while Proposition~\ref{speed} characterizes the speed of convergence up to ball of a certain radius. Beyond that radius nothing can be said about the speed at which the Geometric Fairness Metric converges without further assumptions. However, it is still certain that the metric will eventually converge to zero thanks to Proposition~\ref{Convergence}.

Another important result is the transferral of fairness properties from the continuous dynamical system to its Euler discretization.

\begin{proposition}
    Consider the perturbed system~\eqref{SheafModified} and its corresponding Euler discretization with step size $1$. The resulting sequence is Geometrically Fair.
\end{proposition}

Finally, we characterize the magnitude of the deviation between the Euler discretization of the original and perturbed systems, which is regulated by the parameter $\gamma$.

\begin{proposition}
    Let $x_t$ be the Euler discretization of system~\eqref{SheafDif} and $\tilde{x}_t$ the corresponding discretization of system~\eqref{SheafModified} arising from methods~\eqref{VectorProjection} to~\eqref{SpectralFilter}. Consider the deviation between both sequences, $e_t = \| x_t - \tilde{x}_t\|$. Then, $e_t = O(\gamma)$. 
\end{proposition}
A more detailed discussion of the above results, as well as full proofs for all the claims, can be found in Appendix \ref{app:Theory}.


\begin{figure*}[t]
    \centering
    \includegraphics[width=0.29\linewidth]{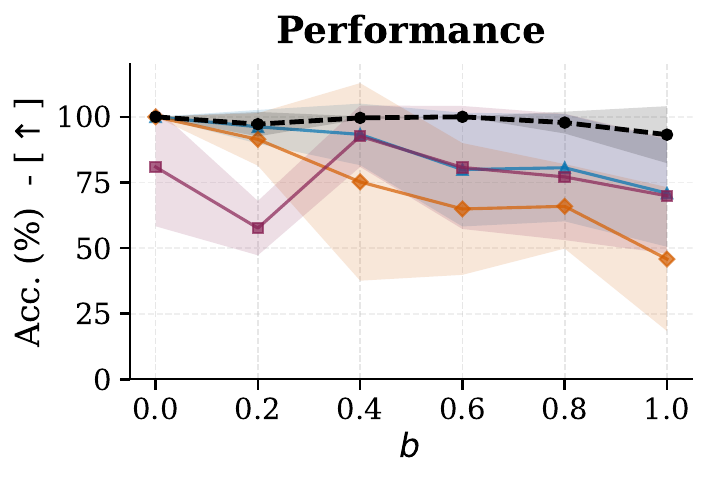}
    \includegraphics[width=0.29\linewidth]{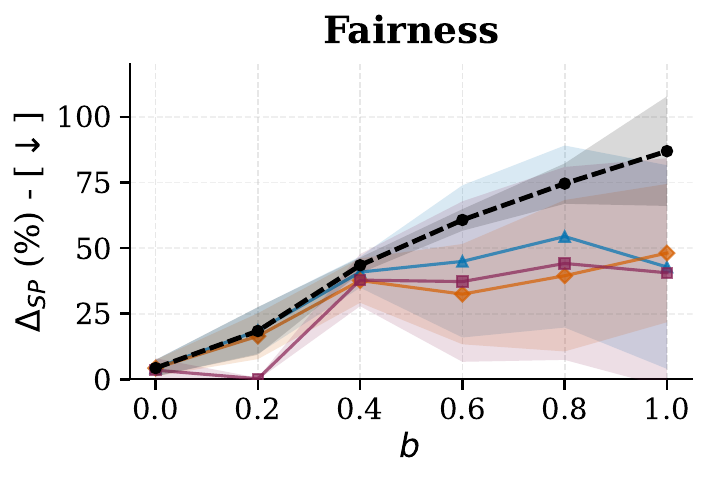}
    \includegraphics[width=0.38\linewidth]{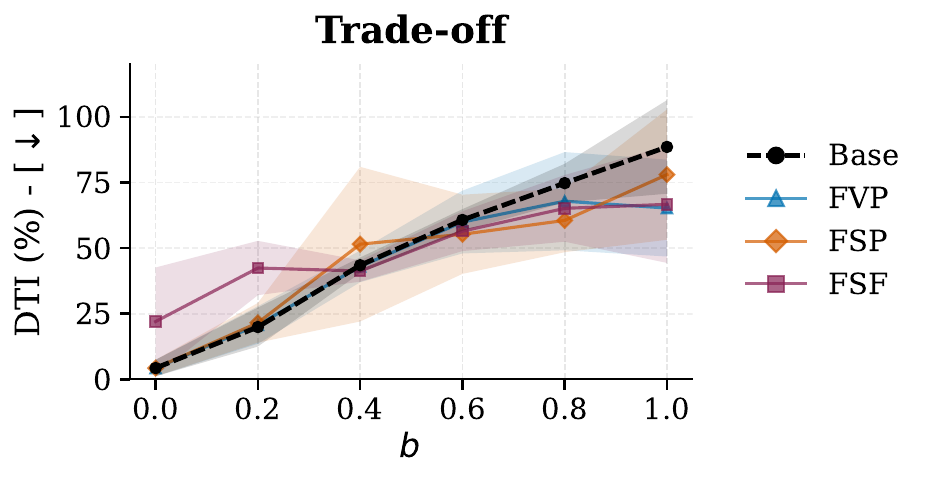}
    \caption{\textbf{Synthetic experiment.} $b$ controls for bias on the data, with $b=0$ representing no bias in the graph and $b=1$ configuring a high-bias setting in which it is difficult to make correct predictions without compromising statistical parity. As bias increases, our models drop performance while maintaining greater fairness, dominating the baseline in the utility-fairness trade-off in the high-bias regime.}
    \label{fig:synthetic}
\end{figure*}

\begin{figure*}[t]
    \centering
    \includegraphics[width=0.3\linewidth]{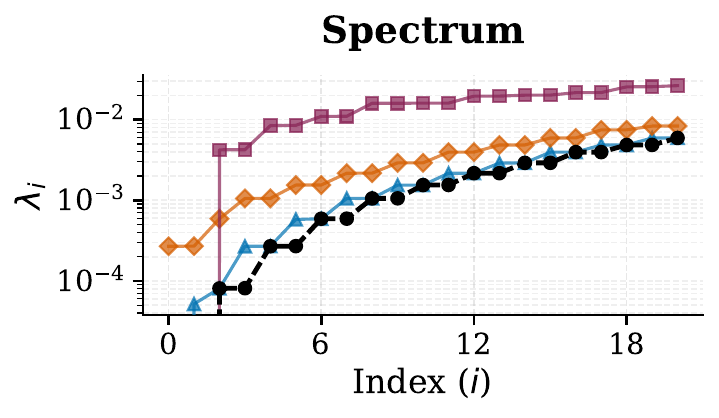}
    \includegraphics[width=0.3\linewidth]{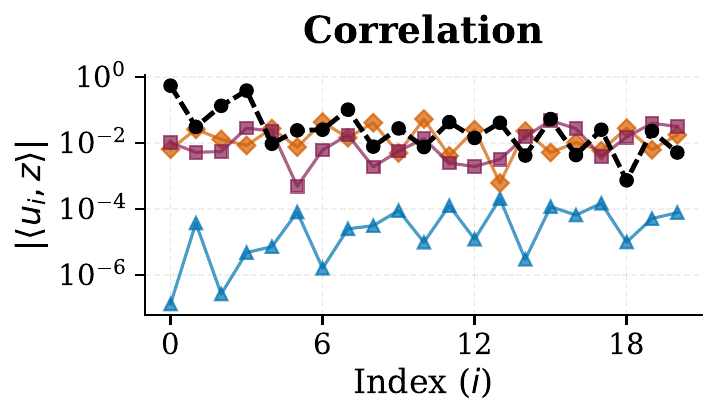}
    \includegraphics[width=0.36\linewidth]{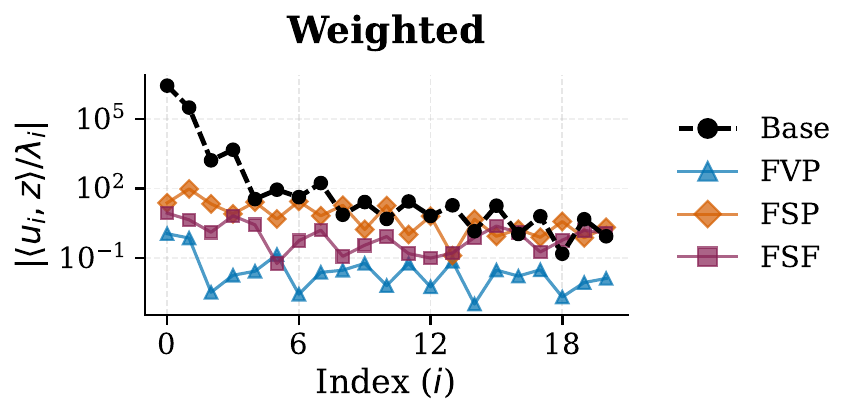}
    \caption{\textbf{First eigenvalues and correlations with the sensitive vector, ordered from smallest to largest eigenvalue.} Smaller eigenvalues dissipate slowly, so low-frequency eigenvectors highly correlated with the sensitive attribute induce long-term bias. The quotient of the correlation divided by the eigenvalues provides a scale on bias dissipation, which our models consistently reduce.}
    \label{fig:spectrum}
\end{figure*}

\begin{figure}[t]
    \centering
    \includegraphics[width=\linewidth]{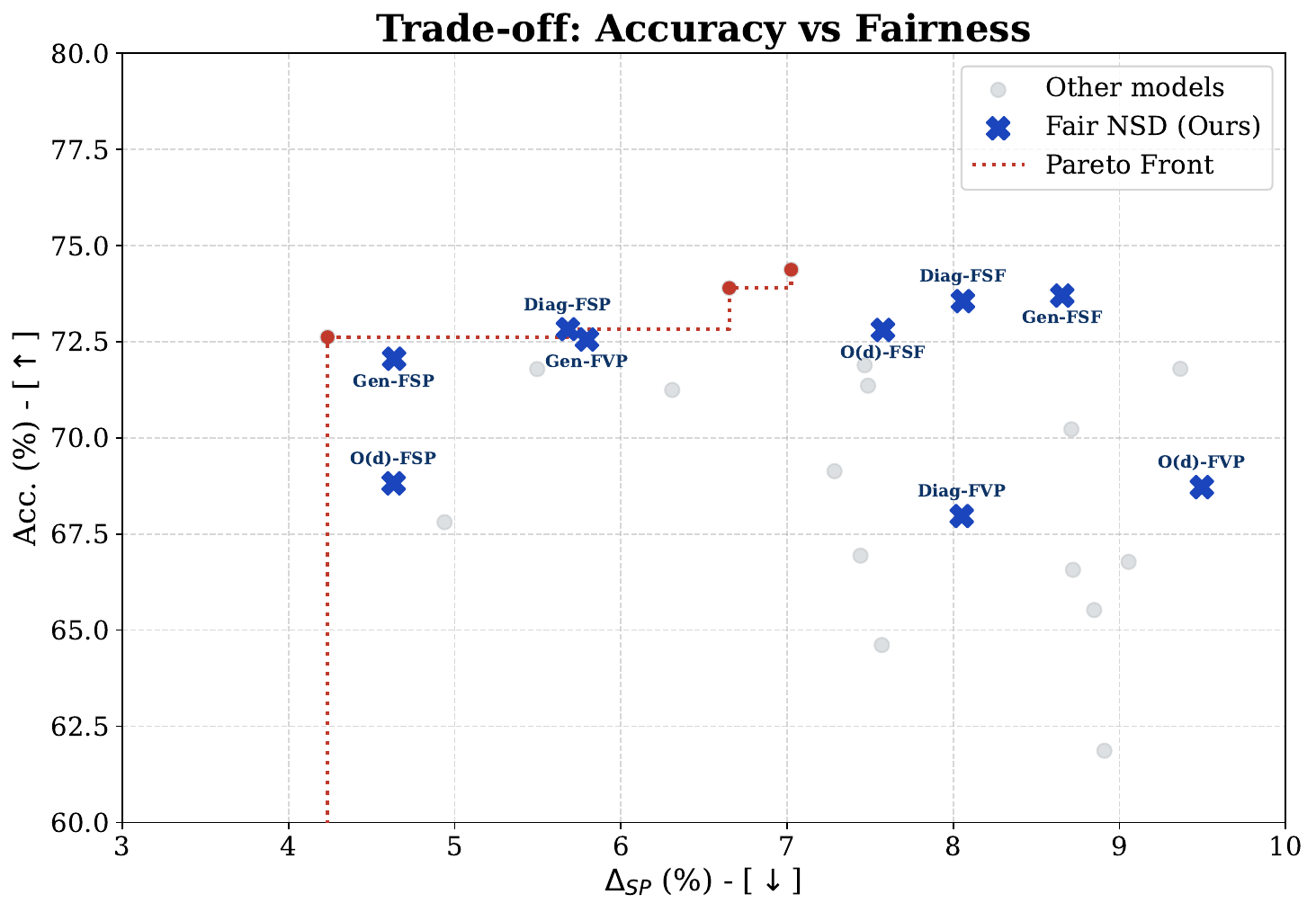}
    \caption{\textbf{Pareto efficient frontier of the methods.} Our models are highlighted in blue, and the Pareto front in red. Our models are close to the frontier, with diagonal models mostly dominating other methods and the general parametrization close behind.}
    \label{fig:pareto}
\end{figure}

\begin{table*}[t]
\centering
\caption{\textbf{Aggregated results for all datasets}. We use accuracy (Acc.) to measure performance, statistical parity ($\Delta_{SP}$) to quantify fairness, and the distance between these two metrics and the ideal point (DTI), $\text{Acc.}=1$, $\Delta_{SP}=0$, to establish the trade-off. The up-arrow ($\uparrow$) indicates that the metric is better the higher it is, and similarly for the down-arrow ($\downarrow$).  The three best results are highlighted in red, blue, and green, respectively.}
\resizebox{0.59\textwidth}{!}{
\begin{tabular}{llccc}
\toprule
& {} & \multicolumn{3}{c}{Metrics} \\
\cmidrule(lr){3-5}
Family & Model & Acc. ($\uparrow$) & $\Delta_{SP}$ ($\downarrow$) & DTI ($\downarrow$) \\
\midrule
\multirow{9}{*}{Fair NSD (Ours)} & Diag-FVP & $66.91 \, {\scriptscriptstyle \pm 4.92}$ & \third{$5.10 \, {\scriptscriptstyle \pm 8.83}$} & $34.10 \, {\scriptscriptstyle \pm 7.04}$ \\
 & O(d)-FVP & $59.40 \, {\scriptscriptstyle \pm 16.43}$ & $10.99 \, {\scriptscriptstyle \pm 9.25}$ & $43.16 \, {\scriptscriptstyle \pm 15.43}$ \\
 & Gen-FVP & $67.75 \, {\scriptscriptstyle \pm 4.76}$ & $17.10 \, {\scriptscriptstyle \pm 21.58}$ & $39.34 \, {\scriptscriptstyle \pm 14.82}$ \\
 & Diag-FSP & $68.20 \, {\scriptscriptstyle \pm 3.81}$ & $7.44 \, {\scriptscriptstyle \pm 10.55}$ & $33.74 \, {\scriptscriptstyle \pm 6.04}$ \\
 & O(d)-FSP & $62.58 \, {\scriptscriptstyle \pm 13.55}$ & $5.54 \, {\scriptscriptstyle \pm 7.57}$ & $38.48 \, {\scriptscriptstyle \pm 13.36}$ \\
 & Gen-FSP & $67.53 \, {\scriptscriptstyle \pm 3.51}$ & \second{$4.69 \, {\scriptscriptstyle \pm 4.17}$} & $33.00 \, {\scriptscriptstyle \pm 3.75}$ \\
 & Diag-FSF & \third{$70.18 \, {\scriptscriptstyle \pm 3.62}$} & $5.99 \, {\scriptscriptstyle \pm 5.29}$ & \third{$30.68 \, {\scriptscriptstyle \pm 4.57}$} \\
 & O(d)-FSF & $63.39 \, {\scriptscriptstyle \pm 9.61}$ & $9.07 \, {\scriptscriptstyle \pm 14.76}$ & $39.69 \, {\scriptscriptstyle \pm 10.90}$ \\
 & Gen-FSF & $67.69 \, {\scriptscriptstyle \pm 1.37}$ & $9.16 \, {\scriptscriptstyle \pm 4.61}$ & $33.78 \, {\scriptscriptstyle \pm 2.54}$ \\
\midrule
\multirow{3}{*}{NSD} & Diag-NSD & $58.48 \, {\scriptscriptstyle \pm 16.95}$ & $11.42 \, {\scriptscriptstyle \pm 15.18}$ & $45.30 \, {\scriptscriptstyle \pm 16.43}$ \\
 & O(d)-NSD & $62.26 \, {\scriptscriptstyle \pm 13.52}$ & $12.75 \, {\scriptscriptstyle \pm 20.48}$ & $43.09 \, {\scriptscriptstyle \pm 16.30}$ \\
 & Gen-NSD & $66.10 \, {\scriptscriptstyle \pm 5.80}$ & $10.41 \, {\scriptscriptstyle \pm 14.21}$ & $37.15 \, {\scriptscriptstyle \pm 9.11}$ \\
\midrule
\multirow{4}{*}{GNN} & GCN & $66.36 \, {\scriptscriptstyle \pm 3.00}$ & $8.57 \, {\scriptscriptstyle \pm 12.56}$ & $36.21 \, {\scriptscriptstyle \pm 5.80}$ \\
 & GAT & $57.79 \, {\scriptscriptstyle \pm 16.53}$ & $11.99 \, {\scriptscriptstyle \pm 18.56}$ & $46.72 \, {\scriptscriptstyle \pm 17.23}$ \\
 & GraphSAGE & $57.78 \, {\scriptscriptstyle \pm 16.40}$ & $12.20 \, {\scriptscriptstyle \pm 20.66}$ & $47.29 \, {\scriptscriptstyle \pm 17.71}$ \\
 & H2GCN & $67.04 \, {\scriptscriptstyle \pm 1.19}$ & $7.50 \, {\scriptscriptstyle \pm 8.59}$ & $34.53 \, {\scriptscriptstyle \pm 3.63}$ \\
\midrule
\multirow{5}{*}{Fair graph} & Fair Drop & $68.31 \, {\scriptscriptstyle \pm 1.96}$ & $9.07 \, {\scriptscriptstyle \pm 12.57}$ & $34.41 \, {\scriptscriptstyle \pm 6.31}$ \\
 & Undersampling & $66.75 \, {\scriptscriptstyle \pm 2.73}$ & \first{$3.73 \, {\scriptscriptstyle \pm 3.75}$} & $33.61 \, {\scriptscriptstyle \pm 2.91}$ \\
 & FairGNN & $67.96 \, {\scriptscriptstyle \pm 2.16}$ & $5.20 \, {\scriptscriptstyle \pm 5.28}$ & $32.77 \, {\scriptscriptstyle \pm 2.62}$ \\
 & FairSIN & $66.48 \, {\scriptscriptstyle \pm 4.15}$ & $8.52 \, {\scriptscriptstyle \pm 13.71}$ & $36.22 \, {\scriptscriptstyle \pm 7.79}$ \\
 & NIFTY & $59.94 \, {\scriptscriptstyle \pm 14.33}$ & $10.09 \, {\scriptscriptstyle \pm 16.80}$ & $43.11 \, {\scriptscriptstyle \pm 17.26}$ \\
\midrule
\multirow{4}{*}{Tabular} & Log. Reg. & $66.26 \, {\scriptscriptstyle \pm 3.71}$ & $11.51 \, {\scriptscriptstyle \pm 12.74}$ & $36.78 \, {\scriptscriptstyle \pm 8.58}$ \\
 & XGBoost & $70.03 \, {\scriptscriptstyle \pm 5.04}$ & $9.24 \, {\scriptscriptstyle \pm 7.09}$ & $31.85 \, {\scriptscriptstyle \pm 6.12}$ \\
 & LightGBM & $69.35 \, {\scriptscriptstyle \pm 3.81}$ & $8.63 \, {\scriptscriptstyle \pm 8.68}$ & $32.52 \, {\scriptscriptstyle \pm 5.91}$ \\
 & MLP & $67.44 \, {\scriptscriptstyle \pm 5.57}$ & $8.47 \, {\scriptscriptstyle \pm 9.93}$ & $34.99 \, {\scriptscriptstyle \pm 3.74}$ \\
\midrule
\multirow{3}{*}{Fair tabular} & Reweighting & \second{$70.30 \, {\scriptscriptstyle \pm 4.47}$} & $5.20 \, {\scriptscriptstyle \pm 5.06}$ & \first{$30.52 \, {\scriptscriptstyle \pm 4.20}$} \\
 & adv-MLP & $67.30 \, {\scriptscriptstyle \pm 2.21}$ & $12.63 \, {\scriptscriptstyle \pm 10.02}$ & $36.15 \, {\scriptscriptstyle \pm 2.75}$ \\
 & Rej. Opt. & \first{$71.44 \, {\scriptscriptstyle \pm 3.14}$} & $8.85 \, {\scriptscriptstyle \pm 6.76}$ & \second{$30.59 \, {\scriptscriptstyle \pm 1.91}$} \\
\bottomrule
\end{tabular}}
\label{tab:results}
\end{table*}

\section{Experiments}
\label{Experiments}
This Section presents the empirical assessment of the proposed methods. Further details on the datasets used can be found in Appendix~\ref{app:Datasets}, and for a more thorough explanation of the training procedure, see Appendix~\ref{app:TrainingDetails}.
\paragraph{Synthetic experiment.} We consider a simple setup given by a Stochastic Block Model~\cite{SBM} with two communities where the label is given by the community each node belongs to, the features are generated as random Gaussian noise, and the sensitive attribute is generated as a Bernoulli variable. A single parameter $b$ in $[0,1]$ controls the bias present in the data, simultaneously modulating the inter-group edge probability, the correlation between the sensitive attribute and the underlying community, and the pairwise correlations between the features, sensitive variable and label. In summary, when $b\rightarrow0$, the inter-group edge probability is considerably larger than the intra-group edge probability, and the sensitive attribute is assigned at random, thus facilitating unbiased community prediction. On the other hand, when $b\rightarrow1$, the inter- and intra-group probabilities almost equal, and there is a strong correlation between the sensitive attribute and the underlying community, thus making fair prediction harder. The results are displayed in Figure ~\ref{fig:synthetic}, which shows our methods barely modify the underlying baseline in the low $b$ regime, but when $b$ is close to one they compromise accuracy to attain better statistical parity, achieving results with lower distance to the ideal point, thus achieving a better balance in the accuracy-fairness trade-off than the baseline.
\paragraph{Spectrum.} We apply our models using the general parametrization in the German dataset to build intuition. Figure ~\ref{fig:spectrum} shows the modifications our methods introduce to the smallest eigenvalues of the Sheaf Laplacian and the correlation of the corresponding eigenvectors. While FVP barely modifies the spectrum, it greatly reduces correlations, thus reducing bias in the resulting solution. On the other hand, FSF and FSP have a greater effect on the spectrum, while leaving the correlations relatively intact. Nonetheless, the final effect on the correlations, weighted by the inverse of the eigenvalues (which gives the magnitude of the overall effect on bias), is similar across all three methods and decreases significantly relative to the base algorithm.
\paragraph{Real-world experiments.} We test our models on a series of real-world datasets standard in the literature of fair Machine Learning over graphs \cite{POKEC, FairGNN, GERMAN, COMPAS2}, more details in Appendix~\ref{app:Datasets}. We perform a thorough comparison with many models present in the literature. Our main baseline is the original NSD \cite{nsd} using all three parametrizations: diagonal (Diag), orthogonal (O(d)) and general (Gen). However, we also compare our models against many standard graph models like Graph Convolutional Networks (GCN) \cite{GCN}, Graph Attention Networks (GAT) \cite{GAT}, GraphSAGE \cite{SAGE} and H2GCN \cite{H2GCN}, and against FairDrop \cite{FairDROP}, undersampling \cite{POKECSMALL}, FairGNN \cite{FairGNN}, FairSIN \cite{FairSIN}, and NIFTY \cite{Nifty} in terms of fairness-aware graph models. We also compare our models against standard tabular benchmarks, namely, logistic regression (Log. Reg.) \cite{lr}, XGBoost \cite{XGBoost}, LightGBM \cite{LGBM}, and a multi-layer perceptron (MLP) \cite{MLP}, and against reweighting ~\cite{Reweighting}, adversarial debiasing (adv-MLP) \cite{ADVERSARIAL}, and the reject option classifier (Rej. Opt.) \cite{RejectOption} in terms of bias mitigation methods. All models are trained for $100$ epochs on a train-val-test split of $50\%/25\%/25\%$ with label reweighting, early stop, and a tolerance of $40$ epochs. After the initial search, a $5-$fold cross-validation is performed to obtain an estimate of the error, followed by a final run on the full train-val split for $400$ epochs. More details on the training procedure can be found in Appendix~\ref{app:TrainingDetails}.

The aggregated results can be found in Table ~\ref{tab:results}. Our models achieve competitive results, placing in the top positions across most metrics. More importantly, our models achieve the top three spots for statistical parity, with Diag-FVP and Diag-FSP doing so at a cost in accuracy between $2\%$ and $3\%$, while Diag-FSF achieves the third best mean accuracy with a respectable statistical parity of $5.64\%$, awarding it the best position in the ranking in terms of DTI, thus showing the efficacy of the approach and how it can achieve better results in the utility-fairness trade-off. This can also be seen in Figure ~\ref{fig:pareto}, highlighting the fact that many of our methods lie close to the Pareto front of statistical parity and accuracy. In general, the O(d) bundle parametrization is the worst-performing among our proposed methods, clearly dominated by other approaches, whereas the diagonal parametrization with FVP and FSF lies exactly on the Pareto front, making them the superior approach among all methods. The general parametrization yields two models close to the efficient frontier, though the FVP performs particularly poorly on fairness. While the error bars prevent us from claiming that our models achieve state-of-the-art performance in fair Machine Learning, they are competitive with both tabular and graph models. A positive side effect of the proposed methods is improved accuracy relative to the standard NSD. However, this could be explained by noting that our methods are prone to increasing the spectral gap of the sheaf Laplacian (see Figure ~\ref{fig:spectrum}), which is known to mitigate oversquashing and improve the quality of message-passing \cite{FOSR}.

Finally, due to space constraints, we refer readers interested in further results to Appendix~\ref{app:Results}. First, we report a broader set of performance and fairness metrics,  where our method remains competitive. Moreover, we also show results for each individual dataset, showcasing that our method remains robust.

\section{Conclusion}
\label{Conclusion}
\paragraph{Conclusions.} In this work, we frame algorithmic fairness in graph models through the lens of dynamical systems and introduce a geometric notion of bias tied to diffusion dynamics. We show that the Laplacian operator, which underlies a broad class of graph models, can encode bias through low-frequency components aligned with a sensitive direction, leading to persistent unfairness under standard diffusion. To address this, we propose three modifications of the Neural Sheaf Diffusion (NSD) Laplacian and develop theoretical guarantees in a time-dependent setting that ensure the decay of bias along diffusion trajectories. Empirically, we validate the approach on controlled synthetic graphs and standard fairness benchmarks, achieving competitive performance even with tabular methods, and enabling a tunable trade-off between predictive utility and fairness. Finally, by analyzing the resulting spectral structure, we provide interpretable diagnostics that clarify how fairness arises from operator-level modifications within the NSD framework.
\paragraph{Limitations and future work.} The main limitation of this work, which serves as a direction for further research, is its emphasis on the Statistical Parity Metric, suggesting that future work should integrate other fairness metrics compatible with the framework, such as kNN consistency or Equal Odds. However, using the latter requires knowledge of the class labels, just as our current framework does for the sensitive attribute, which is another weakness. To address this limitation, future research should introduce uncertainty into the Geometric Fairness Metric, for example, by replacing the indicator functions for demographic attributes with the probabilities provided by a binary classifier.

\section*{Impact statement}

This paper addresses current limitations in graph models to achieve fair results. In particular, we present a flexible methodology that encodes fairness metrics as dot products in an inner product space. This perspective motivates a bias-mitigation algorithm in modifications to the Laplacian matrix, central to many GNN architectures. In particular, applying our ideas to the Neural Sheaf Diffusion framework results in a fair model capable of handling heterophilic data and oversmoothing. The fairness-utility trade-off is largely controlled by a single parameter, which can be tuned to meet the needs of Fair Machine Learning practitioners.

The proposed methods are backed by solid theoretical insights and achieve competitive results in an empirical assessment, enhancing the fairness metrics while preserving accuracy, providing an additional tool to handle bias in Machine Learning. Potential misuse of this model includes the common dangers of fairness processors. In particular, the text considers the statistical parity metric, which is known to be inappropriate in settings like organizational justice and its misuse, either by malice or accident, can create situations like the so-called \emph{glass cliff}~\cite{FairMLBook}. Furthermore, its use requires knowledge of the sensitive attribute, which might be protected and unusable in some contexts. While we believe the proposed models can be applied to achieve a better compromise in fairness-utility, it is also important to underscore the importance of responsible guidelines and safeguards to prevent their misuse.

\bibliographystyle{plainnat} 

\bibliography{bibliography}

@inproceedings{RULE,
  title={Certifying and removing disparate impact},
  author={Feldman, Michael and Friedler, Sorelle A and Moeller, John and Scheidegger, Carlos and Venkatasubramanian, Suresh},
  booktitle={proceedings of the 21th ACM SIGKDD international conference on knowledge discovery and data mining},
  pages={259--268},
  year={2015}
}

@inproceedings{OPTUNA,
  title={Optuna: A next-generation hyperparameter optimization framework},
  author={Akiba, Takuya and Sano, Shotaro and Yanase, Toshihiko and Ohta, Takeru and Koyama, Masanori},
  booktitle={Proceedings of the 25th ACM SIGKDD international conference on knowledge discovery \& data mining},
  pages={2623--2631},
  year={2019}
}

@inproceedings{INCOMPATIBLE,
  title={Fairness metrics: A comparative analysis},
  author={Garg, Pratyush and Villasenor, John and Foggo, Virginia},
  booktitle={2020 IEEE international conference on big data (Big Data)},
  pages={3662--3666},
  year={2020},
  organization={IEEE}
}

@inproceedings{PROXY2,
  title={Fairness constraints: Mechanisms for fair classification},
  author={Zafar, Muhammad Bilal and Valera, Isabel and Rogriguez, Manuel Gomez and Gummadi, Krishna P},
  booktitle={Artificial intelligence and statistics},
  pages={962--970},
  year={2017},
  organization={PMLR}
}

@inproceedings{PROXY1,
  title={Controlling attribute effect in linear regression},
  author={Calders, Toon and Karim, Asim and Kamiran, Faisal and Ali, Wasif and Zhang, Xiangliang},
  booktitle={2013 IEEE 13th international conference on data mining},
  pages={71--80},
  year={2013},
  organization={IEEE}
}

@inproceedings{CONSISTENCY,
  title={Beyond Consistency: Nuanced Metrics for Individual Fairness},
  author={Waller, Madeleine and Rodrigues, Odinaldo and Cocarascu, Oana},
  booktitle={Proceedings of the 2025 ACM Conference on Fairness, Accountability, and Transparency},
  pages={2087--2097},
  year={2025}
}

@article{POKECSMALL,
  title={Benchmarking Fairness-aware Graph Neural Networks in Knowledge Graphs},
  author={Sasaki, Yuya},
  journal={arXiv preprint arXiv:2510.18473},
  year={2025}
}

@inproceedings{POKEC,
  title={Data analysis in public social networks},
  author={Takac, Lubos and Zabovsky, Michal},
  booktitle={International scientific conference and international workshop present day trends of innovations},
  volume={1},
  number={6},
  year={2012}
}

@inproceedings{AGE,
  title={Classifying without discriminating},
  author={Kamiran, Faisal and Calders, Toon},
  booktitle={2009 2nd international conference on computer, control and communication},
  pages={1--6},
  year={2009},
  organization={IEEE}
}

@inproceedings{ADVERSARIAL,
  title={Mitigating unwanted biases with adversarial learning},
  author={Zhang, Brian Hu and Lemoine, Blake and Mitchell, Margaret},
  booktitle={Proceedings of the 2018 AAAI/ACM Conference on AI, Ethics, and Society},
  pages={335--340},
  year={2018}
}

@article{GERMAN,
  title={Fairness in credit scoring: Assessment, implementation and profit implications},
  author={Kozodoi, Nikita and Jacob, Johannes and Lessmann, Stefan},
  journal={European Journal of Operational Research},
  volume={297},
  number={3},
  pages={1083--1094},
  year={2022},
  publisher={Elsevier}
}

@inproceedings{COMPAS2,
  title={A case-based-reasoning analysis of the COMPAS Dataset},
  author={van Woerkom, Wijnand and Grossi, Davide and Prakken, Henry and Verheij, Bart},
  booktitle={37th Annual Conference on Legal Knowledge and Information Systems, JURIX 2024},
  pages={180--190},
  year={2024},
  organization={IOS Press}
}

@misc{COMPAS,
    title = {Machine Bias},
    author={ProPublica},
    howpublished = {\url{http://web.archive.org/web/20080207010024/http://www.808multimedia.com/winnt/kernel.htm}},
    year={2016},
    note = {Accessed: 2025-11-11}
}

@book{Cheby,
  title={Chebyshev polynomials},
  author={Mason, John C and Handscomb, David C},
  year={2002},
  publisher={Chapman and Hall/CRC}
}

@book{BigOEigen,
  title={Numerical methods for large eigenvalue problems: revised edition},
  author={Saad, Yousef},
  year={2011},
  publisher={SIAM}
}

@book{Variation,
  title={An introduction to difference equations},
  author={Elaydi, Saber},
  year={2005},
  publisher={Springer}
}

@article{Barbalat,
  title={Variations on Barb{\u{a}}lat's lemma},
  author={Farkas, B{\'a}lint and Wegner, Sven-Ake},
  journal={The American Mathematical Monthly},
  volume={123},
  number={8},
  pages={825--830},
  year={2016},
  publisher={Taylor \& Francis}
}

@misc{FairSheafDiffusion,
      title={On the use of graph models to achieve individual and group fairness}, 
      author={Arturo Pérez-Peralta and Sandra Benítez-Peña and Rosa E. Lillo},
      year={2026},
      eprint={2601.08784},
      archivePrefix={arXiv},
      primaryClass={stat.ML},
      url={https://arxiv.org/abs/2601.08784}, 
}

@inproceedings{Start3,
author = {Liu, Zheyuan and Zhang, Chunhui and Tian, Yijun and Zhang, Erchi and Huang, Chao and Ye, Yanfang and Zhang, Chuxu},
title = {Fair Graph Representation Learning via Diverse Mixture-of-Experts},
year = {2023},
isbn = {9781450394161},
publisher = {Association for Computing Machinery},
address = {New York, NY, USA},
url = {https://doi.org/10.1145/3543507.3583207},
doi = {10.1145/3543507.3583207},
booktitle = {Proceedings of the ACM Web Conference 2023},
pages = {28–38},
numpages = {11},
location = {Austin, TX, USA},
series = {WWW '23}
}

@inproceedings{Start2,
  author={Kose, O. Deniz and Shen, Yanning},
  booktitle={2022 30th European Signal Processing Conference (EUSIPCO)}, 
  title={Fairness-aware Adaptive Network Link Prediction}, 
  year={2022},
  volume={},
  number={},
  pages={677-681},
  doi={10.23919/EUSIPCO55093.2022.9909546}}

@inproceedings{Start1,
  title={Eqgnn: Equalized node opportunity in graphs},
  author={Singer, Uriel and Radinsky, Kira},
  booktitle={Proceedings of the AAAI conference on artificial intelligence},
  volume={36},
  number={8},
  pages={8333--8341},
  year={2022}
}

@article{Survey2,
  title={Fairness-aware graph neural networks: A survey},
  author={Chen, April and Rossi, Ryan A and Park, Namyong and Trivedi, Puja and Wang, Yu and Yu, Tong and Kim, Sungchul and Dernoncourt, Franck and Ahmed, Nesreen K},
  journal={ACM Transactions on Knowledge Discovery from Data},
  volume={18},
  number={6},
  pages={1--23},
  year={2024},
  publisher={ACM New York, NY}
}

@article{Survey1,
  title={Fairness in graph mining: A survey},
  author={Dong, Yushun and Ma, Jing and Wang, Song and Chen, Chen and Li, Jundong},
  journal={IEEE Transactions on Knowledge and Data Engineering},
  volume={35},
  number={10},
  pages={10583--10602},
  year={2023},
  publisher={IEEE}
}

@article{FairEDIT,
  title={Fairedit: Preserving fairness in graph neural networks through greedy graph editing},
  author={Loveland, Donald and Pan, Jiayi and Bhathena, Aaresh Farrokh and Lu, Yiyang},
  journal={arXiv preprint arXiv:2201.03681},
  year={2022}
}

@inproceedings{Nifty,
  title={Towards a unified framework for fair and stable graph representation learning},
  author={Agarwal, Chirag and Lakkaraju, Himabindu and Zitnik, Marinka},
  booktitle={Uncertainty in artificial intelligence},
  pages={2114--2124},
  year={2021},
  organization={PMLR}
}

@inproceedings{FairSIN,
  title={Fairsin: Achieving fairness in graph neural networks through sensitive information neutralization},
  author={Yang, Cheng and Liu, Jixi and Yan, Yunhe and Shi, Chuan},
  booktitle={Proceedings of the AAAI Conference on Artificial Intelligence},
  volume={38},
  number={8},
  pages={9241--9249},
  year={2024}
}

@inproceedings{FairGNN,
  title={Say no to the discrimination: Learning fair graph neural networks with limited sensitive attribute information},
  author={Dai, Enyan and Wang, Suhang},
  booktitle={Proceedings of the 14th ACM international conference on web search and data mining},
  pages={680--688},
  year={2021}
}

@inproceedings{FairSpecFilter,
  title={Fairness-aware graph filter design},
  author={Kose, O Deniz and Shen, Yanning and Mateos, Gonzalo},
  booktitle={2023 57th Asilomar Conference on Signals, Systems, and Computers},
  pages={330--334},
  year={2023},
  organization={IEEE}
}

@inproceedings{FairSpecRewire,
  title={Filtering as Rewiring for Bias Mitigation on Graphs},
  author={Kose, O Deniz and Mateos, Gonzalo and Shen, Yanning},
  booktitle={2024 IEEE 13rd Sensor Array and Multichannel Signal Processing Workshop (SAM)},
  pages={1--5},
  year={2024},
  organization={IEEE}
}

@article{FairDROP,
  title={Fairdrop: Biased edge dropout for enhancing fairness in graph representation learning},
  author={Spinelli, Indro and Scardapane, Simone and Hussain, Amir and Uncini, Aurelio},
  journal={IEEE Transactions on Artificial Intelligence},
  volume={3},
  number={3},
  pages={344--354},
  year={2021},
  publisher={IEEE}
}

@inproceedings{FOSR,
  title={FoSR: First-order spectral rewiring for addressing oversquashing in GNNs},
  author={Karhadkar, Kedar and Banerjee, Pradeep and Montufar, Guido},
  booktitle={International Conference on Learning Representations},
  year={2023}
}

@article{H2GCN,
  title={Beyond homophily in graph neural networks: Current limitations and effective designs},
  author={Zhu, Jiong and Yan, Yujun and Zhao, Lingxiao and Heimann, Mark and Akoglu, Leman and Koutra, Danai},
  journal={Advances in neural information processing systems},
  volume={33},
  pages={7793--7804},
  year={2020}
}

@article{SBM,
  title={Community detection and stochastic block models: recent developments},
  author={Abbe, Emmanuel},
  journal={Journal of Machine Learning Research},
  volume={18},
  number={177},
  pages={1--86},
  year={2018}
}

@inproceedings{FraudDetection,
  title={Pick and choose: a GNN-based imbalanced learning approach for fraud detection},
  author={Liu, Yang and Ao, Xiang and Qin, Zidi and Chi, Jianfeng and Feng, Jinghua and Yang, Hao and He, Qing},
  booktitle={Proceedings of the web conference 2021},
  pages={3168--3177},
  year={2021}
}

@inproceedings{Recommendation,
  title={Graph convolutional neural networks for web-scale recommender systems},
  author={Ying, Rex and He, Ruining and Chen, Kaifeng and Eksombatchai, Pong and Hamilton, William L and Leskovec, Jure},
  booktitle={Proceedings of the 24th ACM SIGKDD international conference on knowledge discovery \& data mining},
  pages={974--983},
  year={2018}
}

@article{Social,
  title={Enhancing social and collaborative learning using a stacked GNN-based community detection},
  author={Ben Yahia, Nesrine},
  journal={Social Network Analysis and Mining},
  volume={14},
  number={1},
  pages={205},
  year={2024},
  publisher={Springer}
}

@article{ParticlePhysics,
  title={Graph neural networks in particle physics},
  author={Shlomi, Jonathan and Battaglia, Peter and Vlimant, Jean-Roch},
  journal={Machine Learning: Science and Technology},
  volume={2},
  number={2},
  pages={021001},
  year={2020},
  publisher={IOP Publishing}
}

@article{DrugDesign,
  title={Graph neural networks for automated de novo drug design},
  author={Xiong, Jiacheng and Xiong, Zhaoping and Chen, Kaixian and Jiang, Hualiang and Zheng, Mingyue},
  journal={Drug discovery today},
  volume={26},
  number={6},
  pages={1382--1393},
  year={2021},
  publisher={Elsevier}
}

@inproceedings{QuantumChemistry,
  title={Neural message passing for quantum chemistry},
  author={Gilmer, Justin and Schoenholz, Samuel S and Riley, Patrick F and Vinyals, Oriol and Dahl, George E},
  booktitle={International conference on machine learning},
  pages={1263--1272},
  year={2017},
  organization={Pmlr}
}

@article{Opinion,
  title={Opinion dynamics on discourse sheaves},
  author={Hansen, Jakob and Ghrist, Robert},
  journal={SIAM Journal on Applied Mathematics},
  volume={81},
  number={5},
  pages={2033--2060},
  year={2021},
  publisher={SIAM}
}

@article{Diffusion,
  title={Adaptive diffusion in graph neural networks},
  author={Zhao, Jialin and Dong, Yuxiao and Ding, Ming and Kharlamov, Evgeny and Tang, Jie},
  journal={Advances in neural information processing systems},
  volume={34},
  pages={23321--23333},
  year={2021}
}

@inproceedings{
GIN,
title={How Powerful are Graph Neural Networks?},
author={Keyulu Xu and Weihua Hu and Jure Leskovec and Stefanie Jegelka},
booktitle={International Conference on Learning Representations},
year={2019},
url={https://openreview.net/forum?id=ryGs6iA5Km},
}

@inproceedings{GCN,
  title={Semi-supervised learning with graph learning-convolutional networks},
  author={Jiang, Bo and Zhang, Ziyan and Lin, Doudou and Tang, Jin and Luo, Bin},
  booktitle={Proceedings of the IEEE/CVF conference on computer vision and pattern recognition},
  pages={11313--11320},
  year={2019}
}

@inproceedings{
GAT,
title={Graph Attention Networks},
author={Petar Veličković and Guillem Cucurull and Arantxa Casanova and Adriana Romero and Pietro Liò and Yoshua Bengio},
booktitle={International Conference on Learning Representations},
year={2018},
url={https://openreview.net/forum?id=rJXMpikCZ},
}

@article{SAGE,
  title={Inductive representation learning on large graphs},
  author={Hamilton, Will and Ying, Zhitao and Leskovec, Jure},
  journal={Advances in neural information processing systems},
  volume={30},
  year={2017}
}

@article{Oversquashing1,
  title={Over-squashing in graph neural networks: A comprehensive survey},
  author={Akansha, Singh},
  journal={Neurocomputing},
  pages={130389},
  year={2025},
  publisher={Elsevier}
}

@inproceedings{Oversquashing2,
  title={Revisiting over-smoothing and over-squashing using ollivier-ricci curvature},
  author={Nguyen, Khang and Hieu, Nong Minh and Nguyen, Vinh Duc and Ho, Nhat and Osher, Stanley and Nguyen, Tan Minh},
  booktitle={International Conference on Machine Learning},
  pages={25956--25979},
  year={2023},
  organization={PMLR}
}

@article{Oversmoothing1,
  title={Not too little, not too much: a theoretical analysis of graph (over) smoothing},
  author={Keriven, Nicolas},
  journal={Advances in Neural Information Processing Systems},
  volume={35},
  pages={2268--2281},
  year={2022}
}

@article{Oversmoothing2,
  title={Another perspective of over-smoothing: Alleviating semantic over-smoothing in deep GNNs},
  author={Li, Jin and Zhang, Qirong and Liu, Wenxi and Chan, Antoni B and Fu, Yang-Geng},
  journal={IEEE Transactions on Neural Networks and Learning Systems},
  volume={36},
  number={4},
  pages={6897--6910},
  year={2024},
  publisher={IEEE}
}

@article{heterophily1,
  title={Revisiting heterophily for graph neural networks},
  author={Luan, Sitao and Hua, Chenqing and Lu, Qincheng and Zhu, Jiaqi and Zhao, Mingde and Zhang, Shuyuan and Chang, Xiao-Wen and Precup, Doina},
  journal={Advances in neural information processing systems},
  volume={35},
  pages={1362--1375},
  year={2022}
}

@inproceedings{heterophily2,
  title={Graph neural networks with heterophily},
  author={Zhu, Jiong and Rossi, Ryan A and Rao, Anup and Mai, Tung and Lipka, Nedim and Ahmed, Nesreen K and Koutra, Danai},
  booktitle={Proceedings of the AAAI conference on artificial intelligence},
  volume={35},
  number={12},
  pages={11168--11176},
  year={2021}
}

@inproceedings{FairnessHeterophily1,
  title={Unveiling the impact of local homophily on gnn fairness: In-depth analysis and new benchmarks},
  author={Loveland, Donald and Koutra, Danai},
  booktitle={Proceedings of the 2025 SIAM International Conference on Data Mining (SDM)},
  pages={608--617},
  year={2025},
  organization={SIAM}
}

@article{FairnessHeterophily2,
  title={On graph neural network fairness in the presence of heterophilous neighborhoods},
  author={Loveland, Donald and Zhu, Jiong and Heimann, Mark and Fish, Ben and Schaub, Michael T and Koutra, Danai},
  journal={arXiv preprint arXiv:2207.04376},
  year={2022}
}

@article{nsd,
  title={Neural sheaf diffusion: A topological perspective on heterophily and oversmoothing in gnns},
  author={Bodnar, Cristian and Di Giovanni, Francesco and Chamberlain, Benjamin and Lio, Pietro and Bronstein, Michael},
  journal={Advances in Neural Information Processing Systems},
  volume={35},
  pages={18527--18541},
  year={2022}
}

@article{polynsd,
  title={Polynomial Neural Sheaf Diffusion: A Spectral Filtering Approach on Cellular Sheaves},
  author={Borgi, Alessio and Silvestri, Fabrizio and Li{\`o}, Pietro},
  journal={arXiv preprint arXiv:2512.00242},
  year={2025}
}

@article{Academic,
  title={Big data's disparate impact},
  author={Barocas, Solon and Selbst, Andrew D},
  journal={Calif. L. Rev.},
  volume={104},
  pages={671},
  year={2016},
  publisher={HeinOnline}
}

@inproceedings{Geometric2,
  title={A geometric solution to fair representations},
  author={He, Yuzi and Burghardt, Keith and Lerman, Kristina},
  booktitle={Proceedings of the AAAI/ACM Conference on AI, Ethics, and Society},
  pages={279--285},
  year={2020}
}

@inproceedings{Geoffair,
  title={A geometric framework for fairness},
  author={Maggio, Alessandro and Giuliani, Luca and Calegari, Roberta and Lombardi, Michele and Milano, Michela and others},
  booktitle={CEUR workshop proceedings},
  volume={3523},
  pages={1--17},
  year={2023},
  organization={CEUR-WS}
}

@article{Finance,
  title={Fairness in credit scoring: Assessment, implementation and profit implications},
  author={Kozodoi, Nikita and Jacob, Johannes and Lessmann, Stefan},
  journal={European Journal of Operational Research},
  volume={297},
  number={3},
  pages={1083--1094},
  year={2022},
  publisher={Elsevier}
}

@inproceedings{Justice,
  title={Fairness in predicting recidivism score},
  author={Athota, Jaswanth Kiran and Parimi, Kushal Kumar and Teja, Mantena Krishna and Bhavani, Mayuri Ajay and Yamuna Devi, MM},
  booktitle={International Conference on Smart Data Intelligence},
  pages={239--253},
  year={2024},
  organization={Springer}
}

@inproceedings{Healthcare,
  title={Fairness in machine learning for healthcare},
  author={Ahmad, Muhammad Aurangzeb and Patel, Arpit and Eckert, Carly and Kumar, Vikas and Teredesai, Ankur},
  booktitle={Proceedings of the 26th ACM SIGKDD international conference on knowledge discovery \& data mining},
  pages={3529--3530},
  year={2020}
}

@inproceedings{ManyDefinitions,
  title={Translation tutorial: 21 fairness definitions and their politics},
  author={Narayanan, Arvind},
  booktitle={Proc. conf. fairness accountability transp., new york, usa},
  volume={1170},
  pages={3},
  year={2018}
}

@book{FairMLBook,
  title={Fairness and machine learning: Limitations and opportunities},
  author={Barocas, Solon and Hardt, Moritz and Narayanan, Arvind},
  year={2023},
  publisher={MIT press}
}

@inproceedings{industry,
  title={Improving fairness in machine learning systems: What do industry practitioners need?},
  author={Holstein, Kenneth and Wortman Vaughan, Jennifer and Daum{\'e} III, Hal and Dudik, Miro and Wallach, Hanna},
  booktitle={Proceedings of the 2019 CHI conference on human factors in computing systems},
  pages={1--16},
  year={2019}
}

@inproceedings{Reweighting,
  title={Building classifiers with independency constraints},
  author={Calders, Toon and Kamiran, Faisal and Pechenizkiy, Mykola},
  booktitle={2009 IEEE international conference on data mining workshops},
  pages={13--18},
  year={2009},
  organization={IEEE}
}

@inproceedings{DisparateImpactRemover,
  title={Certifying and removing disparate impact},
  author={Feldman, Michael and Friedler, Sorelle A and Moeller, John and Scheidegger, Carlos and Venkatasubramanian, Suresh},
  booktitle={proceedings of the 21th ACM SIGKDD international conference on knowledge discovery and data mining},
  pages={259--268},
  year={2015}
}

@inproceedings{PrejudiceIndexRegularizer,
  title={Fairness-aware classifier with prejudice remover regularizer},
  author={Kamishima, Toshihiro and Akaho, Shotaro and Asoh, Hideki and Sakuma, Jun},
  booktitle={Joint European conference on machine learning and knowledge discovery in databases},
  pages={35--50},
  year={2012},
  organization={Springer}
}

@inproceedings{MetaFair,
  title={Classification with fairness constraints: A meta-algorithm with provable guarantees},
  author={Celis, L Elisa and Huang, Lingxiao and Keswani, Vijay and Vishnoi, Nisheeth K},
  booktitle={Proceedings of the conference on fairness, accountability, and transparency},
  pages={319--328},
  year={2019}
}

@inproceedings{AdversarialDebiasing,
  title={Mitigating unwanted biases with adversarial learning},
  author={Zhang, Brian Hu and Lemoine, Blake and Mitchell, Margaret},
  booktitle={Proceedings of the 2018 AAAI/ACM Conference on AI, Ethics, and Society},
  pages={335--340},
  year={2018}
}

@inproceedings{RejectOption,
  title={Decision theory for discrimination-aware classification},
  author={Kamiran, Faisal and Karim, Asim and Zhang, Xiangliang},
  booktitle={2012 IEEE 12th international conference on data mining},
  pages={924--929},
  year={2012},
  organization={IEEE}
}

@article{EqualOdds,
  title={Equality of opportunity in supervised learning},
  author={Hardt, Moritz and Price, Eric and Srebro, Nati},
  journal={Advances in neural information processing systems},
  volume={29},
  year={2016}
}

@article{lr,
  title={The regression analysis of binary sequences},
  author={Cox, David R},
  journal={Journal of the Royal Statistical Society Series B: Statistical Methodology},
  volume={20},
  number={2},
  pages={215--232},
  year={1958},
  publisher={Oxford University Press}
}

@article{LGBM,
  title={Lightgbm: A highly efficient gradient boosting decision tree},
  author={Ke, Guolin and Meng, Qi and Finley, Thomas and Wang, Taifeng and Chen, Wei and Ma, Weidong and Ye, Qiwei and Liu, Tie-Yan},
  journal={Advances in neural information processing systems},
  volume={30},
  year={2017}
}

@article{MLP,
  title={Learning representations by back-propagating errors},
  author={Rumelhart, David E and Hinton, Geoffrey E and Williams, Ronald J},
  journal={nature},
  volume={323},
  number={6088},
  pages={533--536},
  year={1986},
  publisher={Nature Publishing Group UK London}
}

@inproceedings{XGBoost,
    author = {Chen, Tianqi and Guestrin, Carlos},
    title = {XGBoost: A Scalable Tree Boosting System},
    year = {2016},
    isbn = {9781450342322},
    publisher = {Association for Computing Machinery},
    address = {New York, NY, USA},
    url = {https://doi.org/10.1145/2939672.2939785},
    doi = {10.1145/2939672.2939785},
    booktitle = {Proceedings of the 22nd ACM SIGKDD International Conference on Knowledge Discovery and Data Mining},
    pages = {785–794},
    numpages = {10},
    location = {San Francisco, California, USA},
    series = {KDD '16}
}


\appendix

\section{Proofs of Section 4}
\label{app:Theory}

In this Appendix we provide the theoretical backbone of the paper, proving a series of results which culminate with the properties stated in Section~\ref{GeometricFairness}.

Starting with the notation, $\| \cdot \|$ denotes the Euclidean norm when applied to vectors and the spectral norm when applied to matrices. Furthermore, in order to study the system \eqref{SheafDif} we will abstract away the sheaf formalism and focus instead on the dynamics of a linear Time-Varying System (TVS) in $\mathbb{R}^m$:
\begin{equation}
\begin{cases}
    \dot{x}(t) = - L(t) x(t) \\ x(0) = x_0 \in \mathbb{R}^{m}
\end{cases}
\label{TVS}
\end{equation}

The definitions of Geometric Fairness can be adapted to this new setting in a straightforward manner.

We will begin by establishing a sufficient condition under which the TVS \eqref{TVS} achieves Geometric Fairness. A necessary preliminary result establishes that when the system matrix $L(t)$ is semidefinite positive for all time $t> 0$ the orbits stay bounded:

\begin{lemma}
    Consider the time-varying system \eqref{TVS}. Assume that the operators $L(t)$ are semidefinite positive. Then, the orbits of the system are bounded by the initial condition, $\| x_t\| \leq \|x_0 \|$.
\end{lemma}

\begin{proof}
    The proof reduces to taking the derivative of the Lyapunov function $V(x) = x^{\top}x/2 \geq 0$ along the orbits of the system:
\begin{equation*}
    \dot{V}(t) = \dotvect{x_t}{\dot{x}_t} = - \dotmat{x_t}{x_t}{L(t)} \leq 0,
\end{equation*}
    where the last inequality follows from the fact that the matrices $L(t)$ are semidefinite positive. Therefore, $V(t)$ has negative derivative from all $t$, hence $V(t) \leq V(0)$, completing the proof. 
\end{proof}

\begin{remark}
The sheaf Laplacian is semidefinite positive; therefore, the orbits of the system \eqref{SheafDif} are bounded by the initial condition.
\end{remark}

This last property is a key ingredient in finding a sufficient condition ensuring Geometric Fairness:

\begin{lemma}
    Consider the time varying system \eqref{TVS}. Assume that the operators $L(t)$ are semidefinite positive and uniformly bounded, $\|L(t) \| \leq K$. Furthermore, assume that for all $x\in\mathbb{R}^m$ and $t>0$
    \begin{equation}
        x^{\top} L(t) x \geq \gamma (\dotvect{z}{x})^2
        \label{MasterEq}
    \end{equation}
    for $\gamma > 0$. Then the system is geometrically fair with respect to $z$.
\end{lemma}

\begin{proof}
    The idea of the proof is to leverage Barbalat's Lemma (concretely, Theorem 5 in \cite{Barbalat}) to prove that the geometric fairness metric converges to zero. To do this, this proof starts by establish that $F_z(t)$ is a $L_2$ function by bounding the integral of its square with a Lyapunov function of the system. Afterwards, we show that the time derivative of $F_z(t)$ is an $L_{\infty}$ function. These two statements, along with Barbabalat's Lemma, are enough to prove convergence to zero of $F_z(t)$ and thus Geometric Fairness.
    
    Starting with the $L_2$ bound, consider the Lyapunov function $V(x) = x^{\top}x/2 \geq 0$, which is related with $F_z(t)$ through its time derivative along the system's orbits:
\begin{equation*}
    \dot{V}(t) = \dotvect{x_t}{\dot{x}_t} = - \dotmat{x_t}{x_t}{L(t)} \leq -  \gamma (\dotvect{z}{x_t})^2 = -\gamma F_z(t)^2.
\end{equation*}
    Integrating from $0$ to $T$ yields
\begin{equation*}
    V(T) - V(0) \leq -\gamma \int_0^T F_z(t)^2 \, dt,
\end{equation*}
    therefore,
\begin{equation*}
    0 \leq \int_0^T F_z(t)^2 \, dt \leq \frac{V(0) - V(T)}{\gamma} \leq \frac{V(0)}{\gamma}
\end{equation*}
    for all $T$. This bound in combination with the positivity of the square of the Geometric Fairness Metric, $F_z(t)^2 \geq 0$, is enough to establish that its integral in $(0, \infty)$ exists and is finite. Furthermore, we can take the limit when $t\rightarrow\infty$ to find
\begin{equation*}
    0 \leq \int_0^{\infty} F_z(t)^2 \, dt \leq \frac{V(0)}{\gamma},
\end{equation*}
    thus $F_z \in L_2(0,\infty)$. Now we only need to give a uniform bound on the norm of $\dot{F}_z$, but this is straightforward:
\begin{multline*}
    |\dot{F}_z(t)| = |2(\dotvect{z}{\dot{x}_t})| = |2 (\dotmat{z}{x_t}{L(t)})| \\ \leq 2 \| z\| \|x_t\| \| L(t)\| \leq 2 \| x_0 \| K
\end{multline*}
    using the fact that system's orbits are bounded by the initial condition and the $L(t)$ are uniformly bounded, $\|L(t) \| \leq K$. Thus $\dot{F}_z$ is $L_{\infty}(0,\infty)$.

    Combining that $F_z $ is $L_2(0,\infty)$ and $\dot{F}_z$ is $L_{\infty}(0,\infty)$ along with Barbalat's Lemma completes the proof.
\end{proof}

\begin{remark}
    The normalized sheaf Laplacians are uniformly bounded with spectral norm in the closed interval $[0,2]$. Therefore, the only ingredient needed in order to achieve Geometric Fairness is the inequality \eqref{MasterEq}.
\end{remark}

With this in mind, the modifications \eqref{VectorProjection} to \eqref{SpectralFilter} aim to achieve the inequality \eqref{MasterEq}. In the case of FVP this is straightforward.

\begin{proposition}
    Consider the time-varying system resulting from the perturbation \eqref{VectorProjection}. Assume that the operators $L(t)$ are semidefinite positive and uniformly bounded, $\|L(t) \| \leq K$. Then the system is geometrically fair with respect to $z$.
\end{proposition}

\begin{proof}
    By the previous lemma, it suffices to check the following inequality:
    \begin{equation*}
        \dotmat{x}{x}{(L(t) + \gamma z z^{\top})} = \dotmat{x}{x}{L(t)} + \gamma (\dotvect{z}{x})^2 \geq \gamma (\dotvect{z}{x})^2.
    \end{equation*}
\end{proof}

The case of FSP and FSF is more convoluted. However, note that both methods give control over the biased spectrum, thus permitting an incerase on the biased spectral gap, resulting in the next proposition.

\begin{proposition}
    Consider the time-varying system \eqref{TVS}. Assume that the operators $L(t)$ are semidefinite positive and uniformly bounded, $\|L(t) \| \leq K$. Furthermore, assume that the biased spectrum is uniformly bounded below, $\lambda_{i,t} \geq \lambda_0 > 0$ for all $\lambda_{i,t} \in \bspec L(t)$ and $t > 0$. Then the system is geometrically fair with respect to $z$.
\end{proposition}

\begin{proof}
    The proof follows from straightforward computations. Start by considering the coordinate expression of an arbitrary vector $x$ into the eigenbase of $L(t)$,
    \begin{equation*}
        x = \sum_{i} \dotangle{x}{v_{i,t}} v_{i,t}.
    \end{equation*}
    On the one hand, computing the bilinear form induced by $L(t)$ in this coordinate system yields
    \begin{multline*}
        \dotmat{x}{x}{L(t)} = \sum_{i} \dotangle{x}{v_{i,t}}^2 \lambda_i \\= \sum_{i\in \bspec L(t)} \dotangle{x}{v_{i,t}}^2 \lambda_i + \sum_{i\notin \bspec L(t)} \dotangle{x}{v_{i,t}}^2 \lambda_i,
    \end{multline*}
    and using the fact that $L(t)$ is semidefinite positive
    \begin{multline*}
     \sum_{i\in \bspec L(t)} \dotangle{x}{v_{i,t}}^2 \lambda_i + \sum_{i\notin \bspec L(t)} \dotangle{x}{v_{i,t}}^2 \lambda_i\\ \geq \sum_{i\in \bspec L(t)} \dotangle{x}{v_{i,t}}^2 \lambda_i.
    \end{multline*}
    Finally, using the uniform bound on the biased spectrum we arrive at the first inequality needed for the proof:
    \begin{multline}
        \dotmat{x}{x}{L(t)} \geq \sum_{i\in \bspec L(t)} \dotangle{x}{v_{i,t}}^2 \lambda_i \\ \geq \lambda_0 \sum_{i\in \bspec L(t)} \dotangle{x}{v_{i,t}}^2 .
        \label{ineq1}
    \end{multline}
    On the other hand, we can give an expression for the projection of an arbitrary vector $x$ onto $z$ as a sum ranging over the biased spectrum,
    \begin{equation*}
        (\dotvect{z}{x})^2 = \left(\sum_{i\in\bspec L(t)} \dotangle{z}{v_{i,t}} \dotangle{x}{v_{i,t}}\right)^2.
    \end{equation*}
    By the Cauchy-Schwarz inequality,
    \begin{multline*}
 \left(\sum_{i\in\bspec L(t)} \dotangle{z}{v_{i,t}} \dotangle{x}{v_{i,t}}\right)^2 \\ \leq \left(\sum_{i\in \bspec L(t)} \dotangle{z}{v_{i,t}}^2 \right) \left(\sum_{i\in \bspec L(t)} \dotangle{x}{v_{i,t}}^2 \right), 
    \end{multline*}
    and using the fact that $z$ is a unit vector:
    \begin{equation}
        (\dotvect{z}{x})^2 \leq \sum_{i\in \bspec L(t)} \dotangle{x}{v_{i,t}}^2.
    \label{ineq2}
    \end{equation}
    Putting inequalities \eqref{ineq1} and \eqref{ineq2} together yields
\begin{equation*}
    \dotmat{x}{x}{L(t)} \geq \lambda_0 (\dotvect{z}{x})^2,
\end{equation*}
 completing the proof.
\end{proof}

The last two propositions establish Geometric Fairness for the three different proposed methods. However, in order to finish the description of the convergence properties of the system, we need to give a characterization of the speed of convergence. This is done through the next Lemma.

\begin{lemma}
    Consider the time-varying system \eqref{TVS}. Assume that the operators $L(t)$ are semidefinite positive and uniformly bounded, $\|L(t) \| \leq K$. Consider the Dirichlet energy of the vector $z$ with respect to the operator $L(t)$, $E_z(t) = \dotmat{z}{z}{L(t)}$. If the energy is uniformly bounded below, $0 < E_0 \leq E(t)$ for all $t$, then in the resulting system the Geometric Fairness Metric $F_z(t)$ decays exponentially at a rate of $E_0$ towards a ball of radius $O(K/E_0)$.   
\end{lemma}

\begin{proof}
    The idea for the proof is to study the dynamics of the square root of the Geometric Fairness Metric, that is, the dynamics of the projection of the system onto the vector $z$. This will result in an ODE consisting of two terms given by the exponential decay of the metric at $0$ plus a residual. Bounding the decay rate of this solution and the residual terms allows us to prove the result.
    
    Consider the projection $y(t) = \dotangle{z}{x_t}$ and take its time derivative,
    \begin{equation*}
        \dot{y}(t) = -\dotmat{z}{x_t}{L(t)}.
    \end{equation*}
    Now decompose $x_t$ into its projection onto $z$ and a component orthogonal to $z$, $x_t = \dotangle{z}{x_t}z + x_t^{\perp}$. By substituting into the equation above:
    \begin{equation*}
        \dot{y}(t) = -\dotangle{z}{x_t}\dotmat{z}{z}{L(t)} - \dotmat{z}{x_t^{\perp}}{L(t)} = -E(t) y(t) + R(t),
    \end{equation*}
    where the residual $R$ is given by the expression $R(t) = -\dotmat{z}{x_t^{\perp}}{L(t)}$. Notice that the bounds on the norm of the operators $L(t)$ and the orbits $x_t$ result in the inequality $|R(t)| \leq K\|x_0 \|$, which will be useful later. Returning to the ODE, move the $y$ term to the left-hand side and multiply by the integrating factor $\exp{\left(\int_0^t E(s)ds \right)}$, resulting in
    \begin{equation*}
        \frac{d}{dt}\left[e^{\int_0^t E(s)ds } y(t)  \right] = e^{\int_0^t E(s)ds } R(t).
    \end{equation*}    
    Integrating from $0$ to $t$ yields:
    \begin{equation*}
        y(t) =e^{-\int_0^t E(s)ds } y(0) + \int_{0}^t e^{-\int_{\tau}^t E(s)ds } R(\tau)\, d\tau
    \end{equation*}
    Giving suitable bounds on the integral of the exponential and residual completes the proof. On the one hand:
\begin{equation*}
    E_0 \leq E(t) \Longrightarrow e^{-E_0 t} \geq e^{ -\int_0^t E(t) dt},
\end{equation*}
    which gives us the decay rate. On the other hand, the bound $|R(t)| \leq K\| x_0\|$ implies:
    \begin{multline*}
        \left| \int_{0}^t e^{-\int_{\tau}^t E(s)ds } R(\tau)\, d\tau \right| \\ \leq K\| x_0\| \int_0^t e^{-E_0 (t -\tau)}d\tau = K\| x_0\| \frac{1+e^{-E_0 t}}{E_0},
    \end{multline*}
    which gives us the desired bound on the residual term.
\end{proof}


This Lemma allows us to characterize the speed of convergence of all methods, starting with FVP.

\begin{proposition}
    Consider the TVS system resulting from \eqref{TVS} and the perturbation \eqref{VectorProjection}. Assume that the operators $L(t)$ are semidefinite positive. The induced metric $F_z(t)$ decays exponentially at a rate $\gamma$ towards a ball of radius $O(K/\gamma)$.
\end{proposition}

\begin{proof}
    Follows in a straightforward manner by applying the previous lemma using the following bound:
    \begin{equation*}
        (\dotmat{z}{z}{(L(t) + \gamma z z^{\top} )}) \geq  \gamma \dotangle{z}{z}^2 = \gamma.
    \end{equation*}
\end{proof}

Once again, the situation for FSP and FSF is more convoluted, but it can be managed through the control both methods provide on the biased spectral gap.

\begin{proposition}
    Consider the time-varying system \eqref{TVS}. Assume that the operators $L(t)$ are semidefinite positive and uniformly bounded, $\|L(t) \| \leq K$. Furthermore, assume that the biased spectrum is uniformly bounded below, $\lambda_{i,t} \geq \lambda_0 > 0$ for all $\lambda_{i,t} \in \bspec L(t)$ and $t > 0$. Then, the Geometric Fairness Metric exponentially decays at a rate $\lambda_0$ towards a ball of radius $O(K/\lambda_0)$.
\end{proposition}

\begin{proof}
    Follows in a straightforward fashion from the previous lemma and the following bound:
    \begin{multline*}
        \dotmat{z}{z}{L(t)} = \sum_{i\in\bspec L(t)} \dotangle{z}{v_{i,t} }^2 \lambda_i \\ \geq \lambda_0 \sum_{i\in\bspec L(t)} \dotangle{z}{v_{i,t} }^2 = \lambda_0 \dotangle{z}{z}^2 = \lambda_0.
    \end{multline*}
\end{proof}

The previous results imply Proposition ~\ref{Convergence}, thus establishing the theoretical properties of the FVP, FSP, and FSF methods with respect to fairness.

The last question that could be asked in this regard would be the fairness gain, that is, to characterize the asymptotic properties of the difference of the Geometric Fairness Metrics of the original system and the modification of choice, given by $\Delta F_z(t) = | F_z(t) - \tilde{F}_z(t)|$. However, statements about this difference require further assumptions on the original system matrix, $L(t)$. For example, if $F_z(t) \geq e^{-\lambda_* t} $ for some $\lambda_*$, meaning that that there is some obstruction to fairness in the form of an exponential barrier, it can be shown that $\Delta F_z(t) \geq F_z(0)(e^{-\lambda_* t} - e^{-w  t})$, with $\omega$ the convergence speed of the exponential regime of the modification of choice. However, this requires strong assumptions on the original system matrix, thus restricting the generality of the result.

Instead, we now focus on the deviation of the Euler discretization scheme on the original and modified systems. The proofs are straightforward, with the bounds for FVP and FSP being a consequence of the discrete Variation of Constants Formula, while FSF follows by iterating a simple backstep bound on the deviation term.

\begin{proposition}
\label{Disc1}
    Consider the TVS \eqref{TVS}. Assuming the spectrum of $L(t)$ is bounded between $[0,2]$ for all $t$, the discretization $\tilde{x}(t)$ of the system resulting from the FVP perturbation \eqref{VectorProjection} has the following deviation bound:
    \begin{equation*}
        \| e(t) \| \leq \gamma \sum_{j=0}^t |\dotangle{z}{\tilde{x}_j} |.
    \end{equation*}
\end{proposition}

\begin{proof}
    Using the discrete Variation of Constants Formula (concretely, Theorem 3.17 in \cite{Variation}), the solution of the perturbed system can be expressed as:
    \begin{equation*}
    \tilde{x}_t = x_t - \gamma \sum_{j} \mathcal{D}(t, j+1) \dotangle{z}{\tilde{x}_t} z,
    \end{equation*}
    where $\mathcal{D}(t, j)$ is the diffusion operator between times $j<t$ given by $\prod_{k=j}^t (I - L(k))$. Therefore:
    \begin{multline*}
        \| e_t\| = \| \tilde{x}_t - x_t \| =  \gamma \left\|\sum_{j} \mathcal{D}(t, j+1) \dotangle{z}{x_t} z\right\| \\ \leq \gamma \sum_{j=1}^{t-1} \|\mathcal{D}(t, j+1) \| |\dotangle{z}{\tilde{x}_t}| .
    \end{multline*}
    The bounds on the spectrum of the system matrix $L(t)$ implies that the spectral norm of the diffusion operator $\| \mathcal{D}(t,j+1)\| $ is less than one, therefore,
    \begin{equation*}
    \| e(t) \| \leq \gamma \sum_{j=0}^t |\dotangle{s}{\tilde{x}_j} |, 
    \end{equation*}
    completing the proof.
\end{proof}

\begin{proposition}
        Consider the TVS \eqref{TVS}. Assuming the spectrum of  $L(t)$ is bounded between $[0,2]$ for all $t$, the discretization $\tilde{x}(t)$ of the system resulting from the FSP \eqref{SpectralProjection} results on the following deviation bound:
            \begin{equation*}
        \| e(t) \| \leq \gamma  \sum_{i=0}^t \sum_{j\in J_i}  |\dotangle{v_{i,j}}{\tilde{x}_{i}}|    .
    \end{equation*}
\end{proposition}

\begin{proof}
    The proof is the same as the previous one just changing the projection operators.
\end{proof}

\begin{proposition}
\label{Disc3}
    Consider the TVS \eqref{TVS}. Assume that the spectrum of $I - L(t)$ and $I - P_t(L(t))$ are bounded in $[-\rho,\rho]$ with $\rho \leq 1$. Consider the discretization of $\tilde{x}(t)$ of the system resulting from the FSF perturbation 
    \eqref{SpectralFilter}.  The discretization results on the following deviation bound:
\begin{equation*}
    \| e_{t}\| \leq \gamma \| x_0\| \sum_{j=0}^{t-1}\rho ^j.
\end{equation*}
    Furthermore, if $\rho < 1$ then
\begin{equation*}
    \| e_{t}\| \leq \frac{\gamma \| x_0\|}{1-\rho} ,
\end{equation*}
    and if $\rho = 1$,
\begin{equation*}
    \| e_{t}\| \leq \gamma \|x_0\| t.
\end{equation*}
\end{proposition}

\begin{proof}
    Consider the diffusion operators $\mathcal{D}(t)$ and $\tilde{\mathcal{D}}(t)$ given by $I - L(t)$ and $I - P_t(L(t))$ respectively. In order to bound the deviation term on $t$ perform a backstep,
    \begin{equation*}
        \|e_{t}\| = \|\tilde{x}_{t} - x_{t}  \| = \|\tilde{\mathcal{D}}_{t-1}\tilde{x}_{t-1} - \mathcal{D}_{t-1}x_{t-1}  \|.
    \end{equation*}
Adding and substracting the cross-term $\tilde{\mathcal{D}}_t x_t$ results in
\begin{multline*}
    \|\tilde{\mathcal{D}}_t\tilde{x}_{t} - \mathcal{D}_tx_{t}  \| = \|\tilde{\mathcal{D}}_t\tilde{x}_{t} - \tilde{\mathcal{D}}_t x_t + \tilde{\mathcal{D}}_t x_t - \mathcal{D}_tx_{t} \|\\ \leq \| \tilde{\mathcal{D}}_t \| \| e_t \| + \| \tilde{\mathcal{D}}_t - \mathcal{D}_t \| \|x_t \|.
\end{multline*}
The bound on the spectrum of $I -P(L(t))$ implies in $\|\tilde{\mathcal{D}}_t\|\leq \rho \leq 1$. On the other hand, the bound on the spectrum of $L(t)$ implies the sequence $\{x_t\}_{t\in \mathbb{N}}$ is bounded by the initial condition, $\|x_t \|\leq \| x_0\|$. Finally, the difference of the diffusion operators is bounded by the $\|\tilde{\mathcal{D}}_t - \mathcal{D}_t \| = \|P(L(t)) - L(t) \| = \max_{\lambda_t \in \spec L(t)} \| P_t(\lambda_t) - \lambda_t \|$. In practice, these polynomials approximate the combination of the identity with some Gaussian peaks of height $\gamma$ and localization $\lambda \pm\sigma$ for some eigenvalues $\lambda$ in the biased spectrum. Therefore, the most problematic points are found on the Gaussian peaks, whose heights are controlled by the parameter $\gamma$, giving us $\|\tilde{\mathcal{D}}_t - \mathcal{D}_t \| \leq \gamma$. Therefore,
\begin{equation*}
    \| e_{t}\| \leq \rho \| e_{t-1}\| + \gamma \| x_0\|,
\end{equation*}
and applying this inequality iteratively yields
\begin{equation*}
    \| e_{t}\| \leq \rho^{t} \| e_0\| + \gamma \| x_0\|\sum_{j=0}^{t-1}\rho ^j .
\end{equation*}
The initial condition is the same for both systems, therefore $\|e_0\| = 0$ and
\begin{equation*}
    \| e_{t}\| \leq \gamma \|x_0 \| \sum_{j=0}^{t-1}\rho ^j .
\end{equation*}
Finally, the expression for $\rho < 1$ follows from the geometric series,
\begin{equation*}
    \| e_{t}\| \leq \frac{\gamma \| x_0\|}{1 - \rho},
\end{equation*}
and the expression for $\rho = 1$ follows by direct substitution,
\begin{equation*}
    \| e_{t}\| \leq \gamma \|x_0 \| t.
\end{equation*}
\end{proof}

Propositions~\ref{Disc1} through~\ref{Disc3} can be summarized with the following result:
\begin{proposition}
 \label{DiscreteProperties}
    The discretization $\tilde{x}_t$ of the system resulting from perturbations  \eqref{VectorProjection} to \eqref{SpectralFilter} has the following deviation bounds:
    \begin{align*}
        &\textbf{FVP} \Longrightarrow\| e_t \| \leq \gamma \sum_{i=0}^t |\dotangle{z}{\tilde{x}_{i}} | \\
        &\textbf{FSP} \Longrightarrow         \| e_t \| \leq  \gamma \sum_{i=0}^t \sum_{j\in J_i}  |\dotangle{v_{i,j}}{\tilde{x}_{i}}| \\
        &\textbf{FSF} \Longrightarrow     \| e_{t}\| \leq \gamma \|x_0 \| t
    \end{align*}
    with $\gamma$ playing the role of the magnitude of the projection in FVP, the coefficient of the projection in FSP and the peak height in FSF.
\end{proposition}

It is clear then that in all cases the deviation is $e_t = O(\gamma)$ as noted in the main text. One last remark: depending on the convergence of $\tilde{x}_t$, the deviation terms for FVP and FSP are bounded, while the inequality given for FSF is not. Note that, for FSF, we could have used the tighter bound $\| e_t\| \leq \gamma \sum_{j=0}^{t-1} \|x_j\| \rho^j$. However, this inequality is fundamentally different from those given for FVP and FSP. Crucially, it depends on the norm of the solution of the original system, $\| x_j\|$, which may not even converge to zero, while the bounds for FVP and FSP depend on the projection of the solution onto the biased subspace and eigenvectors, respectively. We have tighter control over these projections through the $\gamma$ parameter, which regulates the speed of convergence.

In any case, one possible solution for FSF (also applicable to the other two methods) is to multiply the diffusion matrix $I - \mtsheaflap$ by a reduction coefficient $\rho \in (0, 1)$, resulting in the uniform bound $\gamma \| x_0\|/(1 - \rho)$. This concludes the exposition of the theoretical properties of Fair Neural Sheaf Diffusion.

\section{Datasets}
\label{app:Datasets}

\subsection{Synthetic experiment}

\begin{table}[]
    \centering
    \caption{Parameters used for the modified SBM in the synthetic experiment and their relationship with the bias $b$.}
    \begin{tabular}{cc}
    \toprule
     Parameter & Value \\
     \midrule
     $k$ & $2$ \\
    $n_1$ & $500$\\
    $n_2$ & $500$\\
    $P_{1,2}$ & $0.04$\\
    $P_{1,1}$ & $0.06 + 0.1(1-b/2)$\\
    $Y(u)$ & $C(u)$ \\
    $p_Z$ &  $(1+b)/2$ \\
    $\delta_Y$ & $2(1-b)$ \\
    $\delta_Z$ & $2b$ \\
    \bottomrule
    \end{tabular}
    \label{tab:SBM}
\end{table}

The data used in the synthetic experiment stems from a modified Stochastic Block Model (SBM). The original SBM is a generative graph model parametrized by a number of communities $k$, the number of nodes per community $n_i$, and a symmetric probability matrix $P$, with $k$ in $\mathbb{N}$, $n_i$ in $\mathbb{N}$ for $i$ from $1$ to $k$, and $P$ in $\mathbb{R}^{k \times k}$. We use the notation $C(u)$ to denote the community of the node $u$, and $V$ to denote the node set. The SBM then generates an undirected graph on $V$ by adding the edge $(u,v)$ with probability $P_{C(u), C(v)}$ for any two nodes $u,v$ in $V$. For a more thorough treatment of SBM, consult \cite{SBM}.

This paper considers an SBM with two communities of $500$ nodes each, using the previous notation with $k=2$ and $n_1 = n_2 = n = 500$. Furthermore, we fix the inter-group probability $P_{1,2} = P_{2,1}$ to $0.04$. The only remaining parameter to specify is the intra-group edge probability, which is regulated by a parameter $b$ in $[0,1]$ controlling the underlying bias of the dataset. In particular, the inter-group probabilities are given by $P_{1,1} = P_{2,2} = P =  0.06 + 0.1(1-b/2)$. To understand how $b$ and $P$ relate to bias we first need to explain how to integrate the resulting graph into the framework of fair node classification. In order to do this, we need node features, labels, and a sensitive attribute. For starters, the labels are given by the community of each node, $Y(u) = C(u)$, thus resulting in a natural binary label depending on the underlying topology. The sensitive attribute is also given by the community, although we introduce a flip probability $p_{Z} = \frac{1-b}{2}$ which can break the correlation between $Y$ and $Z$ depending on $b$. Finally, the node features are generated from independent standard normal distributions plus a contribution from the label and the sensitive attribute, $X \sim \mathcal{N}(0, I) + \delta_{Y} \mathbbm{1}(Y=1) + \delta_{Z} \mathbbm{1}(Z=1)$, thus introducing a mechanism through which the features can be informative of the label and contain a certain ammount of bias. The magnitude of these contributions are also controled by $b$ through the relations $\delta_Y = 2(1-b)$, $\delta_Z = 2b$. Finally, using this specification we generate $32$ features. A summary of the chosen parameters can be found in Table~\ref{tab:SBM}.

To close off this discussion on the synthetic experiment, we explain the intuition on how the $b$ parameter controls the bias of the underlying dataset. On one hand, note that as $b$ gets closer to $0$, the inter- and intra-group probabilities drift apart, thus making the classification problem easier. On the other hand, the flip probability $p_Z$ becomes $0.5$, making the sensitive attribute independent from the community, and the features $X$ become highly correlated with $Y$ but not $Z$. Therefore, training a GNN on the resulting graph will yield a fair model with high probability. However, in the high-bias regime when $b$ gets closer to $1$, the inter- and intra- group probabilities get closer together, thus making it harder for a GNN to distinguish between nodes from either community. Furthermore, the flip probability becomes $0$, thus creating a high correlation between the label and the sensitive attribute, and the features $X$ also become highly correlated with $Z$, thus impeding the task of fair classification. The parameter $b$ allows us to interpolate between these two scenarios, controlling the bias with precision. 

\begin{table*}[t]
\centering
    \caption{\textbf{Characteristics of the datasets used.} NBA and German configure the most heterophilic graphs out of the considered datasets.}
    \begin{tabular}{cccccccccc}
    \toprule
      Dataset & Size & Edges & Features & $Y$ rate & $Y$ homophily & $Z$ rate  & $Z$ homophily & $Corr(Y, Z)$ \\
    \midrule
    NBA & $403$ & $21242$ & $97$ & $0.395$ & $0.508$ & $0.734$ & $0.724$ & $-0.025$\\
    Pokec-n & $6185$ & $30642$ & $267$ & $0.445$ & $0.589$ & $0.653$ & $0.954$ & $0.032$\\
    Pokec-z & $7659$ & $41100$ & $278$ & $0.497$ & $0.585$ & $0.633$ & $0.958$ & $0.031$\\
    \midrule
    German & $1000$ & $11392$ & $63$ & $0.700$ & $0.597$ & $0.149$ & $0.746$ & $0.100$\\
    Compas & $6172$ & $85436$ & $410$ & $0.455$ & $0.888$ & $0.486$ & $0.850$ & $0.141$\\
    \bottomrule
    \end{tabular}
    \label{tab:datasets}
\end{table*}

\subsection{Real-world datasets}

The rest of the experiments are performed on a series of real-world datasets standard in the practice of Fair Machine Learning over graphs, summarized below:
\begin{itemize}
    \item \textbf{NBA}\footnote{\tiny\url{https://www.kaggle.com/datasets/noahgift/social-power-nba}}: This dataset provides a list of statistics on $400$ NBA players, which was further extended to include graph data based on the twitter interactions of the players~\cite{FairGNN}. The classification task is to predict whether each player earns more than the median salary, and the sensitive attribute is nationality (USA or overseas).
    \item \textbf{Pokec}\footnote{ \tiny\url{https://snap.stanford.edu/data/soc-Pokec.html}}: This dataset compiles data of Pokec, one of the biggest social networks in Slovakia \cite{POKEC}. There are two versions sampled by province (Pokec-n, Pokec-z), although we use the downsampled versions used by \citealt{POKECSMALL} due to time constraints. Following \citealt{FairGNN}, we use the region as a sensitive attribute and try to predict the working field of the users.
\end{itemize}
Another common practice in fairness over graphs is to create a synthetic graph from tabular data. In our case, we have chosen to use the kNN graph with $10$ neighbors using the euclidean distance with the following datasets:
\begin{itemize}
    \item \textbf{German}\footnote{\tiny \url{https://archive.ics.uci.edu/dataset/144/statlog+german+credit+data}}: This dataset is comprised of data belonging to $1000$ individuals who are labeled according to their credit risk. Following \citealt{AGE}, we use age as a sensitive variable, considering an individual privileged if he or she is older than $25$.
    \item \textbf{Compas}\footnote{\tiny\url{https://github.com/propublica/compas-analysis}}: This dataset provides a list of demographic data of criminal offenders being labeled by whether they recidivated within two years or not. Following \citealt{COMPAS}, we use race as a sensitive variable, considering non-African-American individuals privileged.
\end{itemize}
The relevant characteristics for each dataset are shown in Table ~\ref{tab:datasets}.

\section{Training details}
\label{app:TrainingDetails}

This Appendix delves into the training details of all the different experiments. All training was done using the AdamW optimizer and performed on an NVIDIA Tesla T4 GPU with $16$GB of VRAM. The system was running NVIDIA driver version $565.57.01$ and CUDA $12.7$. Unless stated otherwise, we use a learning rate of $3\cdot 10^{-3} $, a weight decay of $10^{-4}$ and momentum parameters $\beta_1 = 0.9$, $\beta_2 = 0.9999$.

\paragraph{Synthetic experiment.} We use the diagonal parametrization with the hyperparameters shown in Table~\ref{tab:param_synth}. We generate a grid of bias levels between $0$ and $1$ with a stepsize of $0.2$ and, for each value of $b$, we use five different seeds to generate five different datasets. We train all models for $200$ epochs using early stop with a patience of $30$ epochs, using a $80\%/20\%$ train-test split with a learning rate of $0.01$.

\begin{table}[t]
    \centering
    \caption{Hyperparameter configuration for the \textbf{synthetic experiment.}}
    \label{tab:param_synth}
    \begin{tabular}{llr}
        \toprule
        Category & Parameter & Value \\
        \midrule
        \multirow{10}{*}{Architecture} 
          & Hidden Channels & $64$ \\
          & Bundle Dimension & $2$ \\
          & Layers & $5$ \\
          & Sheaf Activation & Identity \\
          & Dropout (Hidden) & $0.5$ \\
          & Dropout (Input) & $0.1$ \\
          & Left weights & No \\
          & Right weights & No \\
          & Second linear & No \\
          & Sparse Learner & Yes \\
        \midrule
        \multirow{1}{*}{FVP} 
          & $\gamma$ & $1.5$ \\
        \midrule
        \multirow{3}{*}{FSP} 
          & N. Eigenvalues & $200$ \\
          & Selected Frequencies & Top $5$ \\
          & $\gamma$ & $0.2$ \\
        \midrule
        \multirow{5}{*}{FSF} 
          & N. Eigenvalues & $100$ \\
        & Selected Frequencies & Top $1$ \\
          & $\gamma$ & $0.3$ \\
          & Polynomial degree & $71$ \\
          & $\sigma$ & $10^{-3}$ \\
        \bottomrule
    \end{tabular}
\end{table}

\begin{table}[t]
    \centering
    \caption{Hyperparameter configuration for the \textbf{experiments on the spectrum and energy} on the German dataset.}
    \label{tab:param_ger}
    \begin{tabular}{llr}
        \toprule
        Category & Parameter & Value \\
        \midrule
        \multirow{10}{*}{Architecture} 
          & Hidden Channels & $32$ \\
          & Bundle Dimension & $3$ \\
          & Layers & $4$ \\
          & Sheaf Activation & Identity \\
          & Dropout (Hidden) & $0.5$ \\
          & Dropout (Input) & $0.05$ \\
          & Left weights & Yes \\
          & Right weights & Yes \\
          & Second linear & Yes \\
          & Sparse Learner & Yes \\
        \midrule
        \multirow{1}{*}{FVP} 
          & $\gamma$ & $30b$ \\
        \midrule
        \multirow{3}{*}{FSP} 
          & N. Eigenvalues & $100$ \\
          & Selected Frequencies & Top $10$ \\
          & $\gamma$ & $10b$ \\
        \midrule
        \multirow{5}{*}{FSF} 
          & N. Eigenvalues & $100$ \\
        & Selected Frequencies & Top $1$ \\
          & $\gamma$ & $0.6b$ \\
          & Polynomial degree & $71$ \\
          & $\sigma$ & $10^{-3}$ \\
        \bottomrule
    \end{tabular}
\end{table}

\begin{table}[t]
    \centering
    \caption{Hyperparameter configuration for the \textbf{sensitivity analysis} on the Pokec-z dataset.}
    \label{tab:param_pokec}
    \begin{tabular}{llr}
        \toprule
        Category & Parameter & Value \\
        \midrule
        \multirow{10}{*}{Architecture} 
          & Hidden Channels & $8$ \\
          & Bundle Dimension & $3$ \\
          & Layers & $2$ \\
          & Sheaf Activation & ELU \\
          & Dropout (Hidden) & $0.5$ \\
          & Dropout (Input) & $0.1$ \\
          & Left weights & Yes \\
          & Right weights & Yes \\
          & Second linear & Yes \\
          & Sparse Learner & Yes \\
        \midrule
        \multirow{1}{*}{FVP} 
          & $\gamma$ & $[0, 60]$ \\
        \midrule
        \multirow{3}{*}{FSP} 
          & N. Eigenvalues & $200$ \\
          & Selected Frequencies & Top $3$ \\
          & $\gamma$ & $[0, 40]$ \\
        \midrule
        \multirow{5}{*}{FSF} 
          & N. Eigenvalues & $100$ \\
        & Selected Frequencies & Top $3$ \\
          & $\gamma$ & $[0,3]$ \\
          & Polynomial degree & $71$ \\
          & $\sigma$ & $10^{-4}$ \\
        \bottomrule
    \end{tabular}
\end{table}

\paragraph{Spectrum and energy.} For the smaller experiments in the German dataset, namely, those involving the calculations of the sheaf Laplacian spectrum in Section~\ref{Experiments} and the Dirichlet energy in Appendix~\ref{app:Results}, we use the general parametrization with the hyper-parameters shown in~\ref{tab:param_ger}. We choose the general parametrization of the Neural Sheaf Diffusion, using a $70\%/30\%$ train-test split, training the models for $100$ epochs.

\paragraph{Sensitivity analysis.} For the sensitivity analysis performed in Appendix~\ref{app:Results}, we use the general parametrization with the hyper-parameters shown in~\ref{tab:param_pokec}. We use a $85\%/15\%$ train-test split running for 200 epochs with early stop and a patience of $40$ epochs.

\paragraph{Real-world datasets and NSD ablation.} For the real-world datasets we perform an extensive random-search with the parameters shown in Table~\ref{tab:grid_search}, with $200$ iterations using the \emph{Optuna} software \cite{OPTUNA}. We perform a $50\%/25\%/25\%$ train-val-test split, and run each model for $100$ epochs with early stop and a patience of $40$ epochs to obtain preliminary results on the validation set. In accordance with the \emph{$80\%$ rule} \cite{RULE}, we select the model with the highest accuracy in the validation set with Statistical Parity under $20\%$. If there was no model fulfilling this fairness constraint, we simply use the model with the highest accuracy. After selecting the best model, we perform $5$-fold cross-validation on the train-val split to obtain a measurement of the error, followed with a final run of $400$ epochs on the whole train-val split evaluated against the test set.

The ablation study on the NSD architecture in Appendix~\ref{app:Results} follows the previously outlined methodology limited to the Pokec-n dataset with restrictions on the stalk dimension ($d=1$) and left weights (set to false) to emulate the GCN and GAT architectures using Sheaf Diffusion.

\begin{figure*}[!ht]
    \centering
    \includegraphics[width=0.21\linewidth]{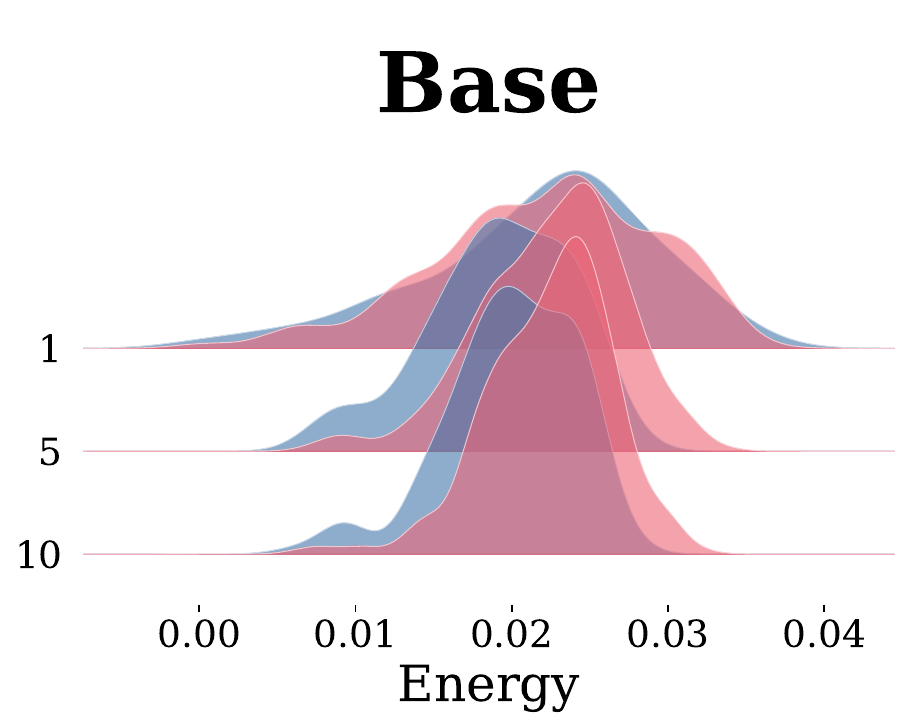}
    \includegraphics[width=0.21\linewidth]{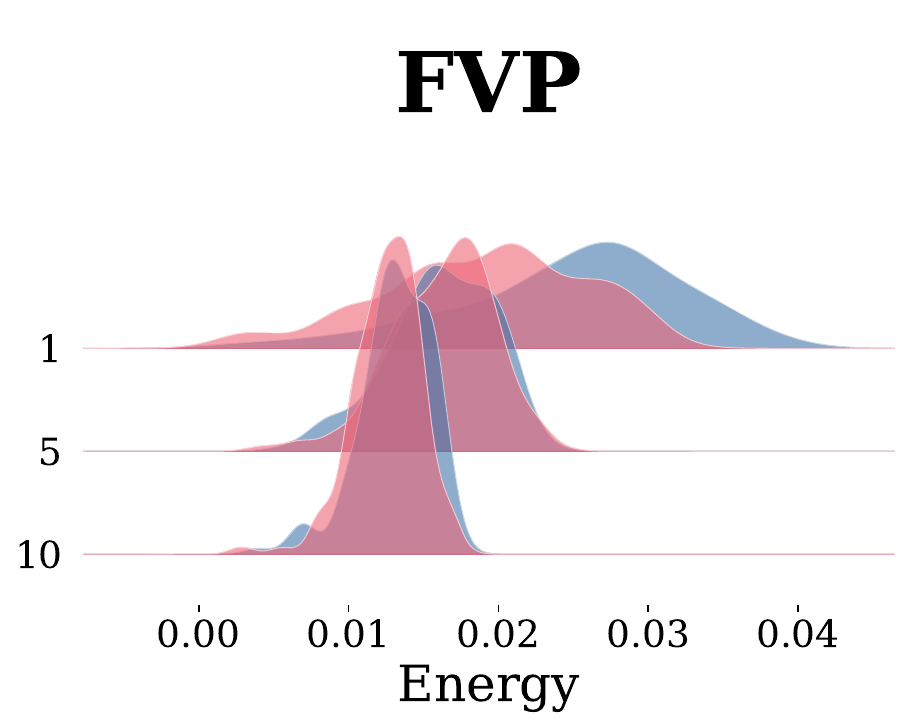} 
    \includegraphics[width=0.21\linewidth]{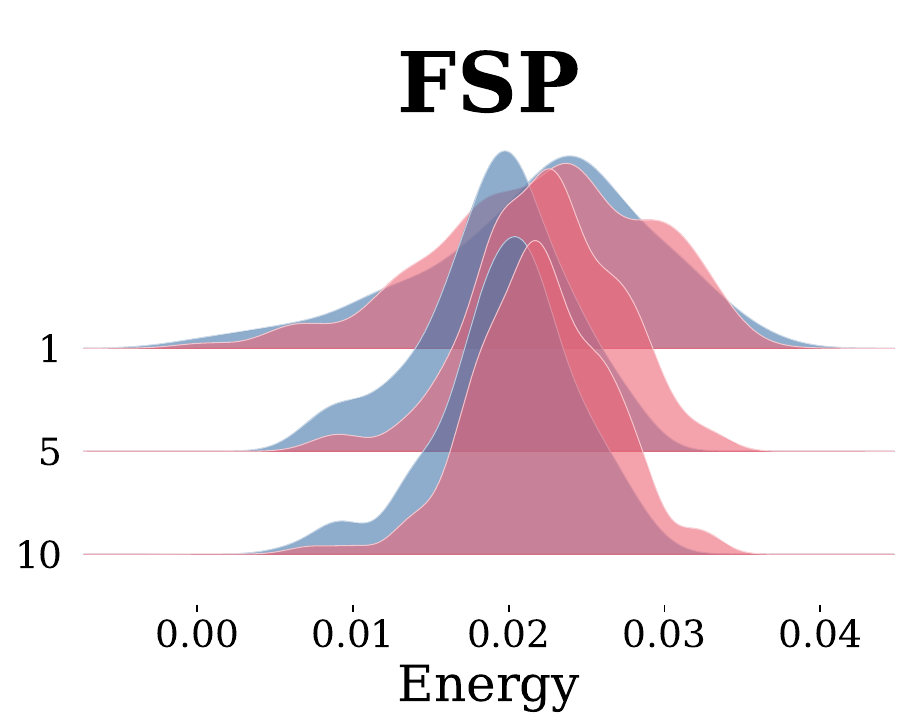} 
    \includegraphics[width=0.31\linewidth]{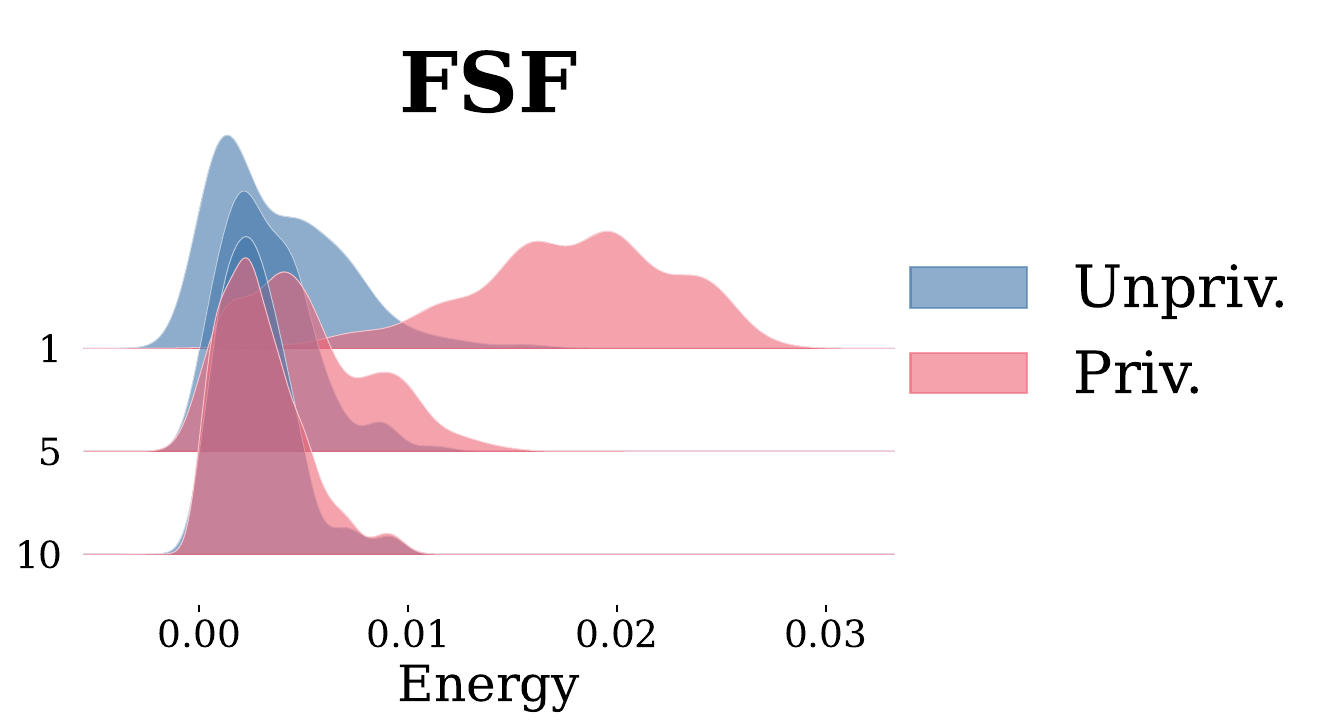}
    
    \caption{Ridge plots of the Dirichlet energies of the original NSD and the different interventions. Our models close the distance between the distributions of both sensitive groups. This plot shows that FSP is the method that modifies the underlying Laplacian the least.}
    \label{fig:ridge}
\end{figure*}

\begin{figure}[!ht]
    \centering
    \includegraphics[width=0.4\textwidth]{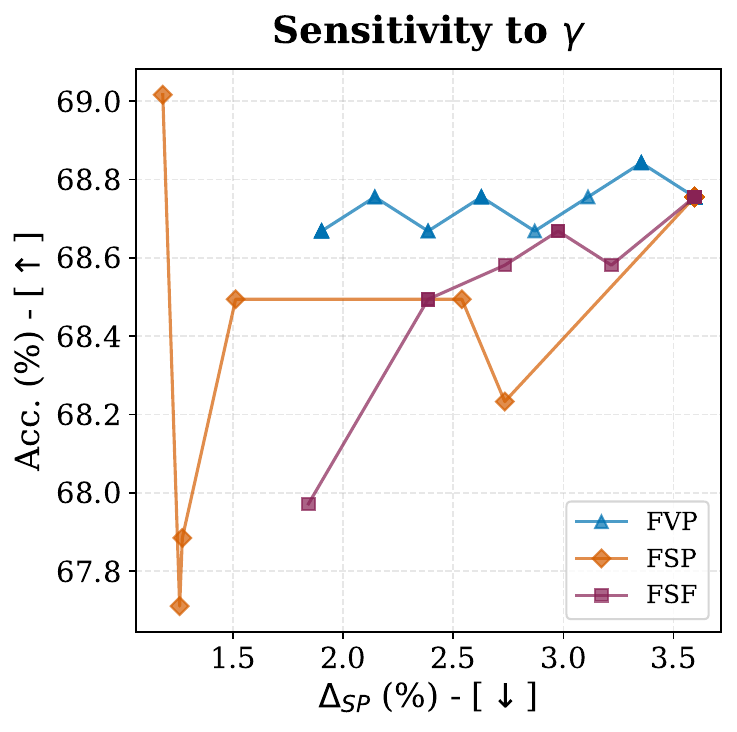}
    \caption{Accuracy-Fairness trade-offs available when varying the $\gamma$ parameter. All modifications start from the same point when $\gamma=0$. As $\gamma$ increases, fairness tends to increase and accuracy tends to decrease. However, the effect of modifications on the spectral gap can lead to spontaneous improvements in accuracy due to its relationship to oversquashing.}
    \label{fig:tradeoff}
\end{figure}

\begin{table*}[t]
        \caption{Results for the Laplacian modifications on \textbf{GCN} for the \textbf{NBA} dataset.}
    \centering
\begin{tabular}{lccccccc}
\toprule
 & \multicolumn{7}{c}{Metrics} \\
 \cmidrule(lr){2-8}  & \multicolumn{3}{c}{Utility} & \multicolumn{3}{c}{Fairness} & Trade-off \\
\cmidrule(lr){2-4} \cmidrule(lr){5-7} \cmidrule(lr){8-8} 
  Model &  Acc. ($\uparrow$) &  F1 ($\uparrow$) & AUC ($\uparrow$) & $\Delta_{SP}$ ($\downarrow$) & $\Delta_{EO}$ ($\downarrow$) &  $\text{Con}_{k=10} $ ($\uparrow$) &  DTI ($\downarrow$) \\
\midrule
GCN (Base)     &  $61.25 \, {\scriptscriptstyle \pm 2.85}$ &  $56.49 \, {\scriptscriptstyle \pm 5.54}$ &  $66.05 \, {\scriptscriptstyle \pm 2.75}$ &  $11.34 \, {\scriptscriptstyle \pm 8.23}$ &  $19.05 \, {\scriptscriptstyle \pm 13.17}$ &  $68.05 \, {\scriptscriptstyle \pm 4.81}$ &  $41.06 \, {\scriptscriptstyle \pm 4.49}$ \\
\midrule
GCN-FVP &  $62.25 \, {\scriptscriptstyle \pm 3.10}$ &  $55.30 \, {\scriptscriptstyle \pm 4.65}$ &  $68.37 \, {\scriptscriptstyle \pm 1.90}$ &   $8.06 \, {\scriptscriptstyle \pm 8.20}$ &  $35.26 \, {\scriptscriptstyle \pm 11.05}$ &  $72.67 \, {\scriptscriptstyle \pm 4.29}$ &  $39.43 \, {\scriptscriptstyle \pm 3.44}$ \\
GCN-FSP &  $62.75 \, {\scriptscriptstyle \pm 0.94}$ &  $56.19 \, {\scriptscriptstyle \pm 2.19}$ &  $68.31 \, {\scriptscriptstyle \pm 1.23}$ &   $9.37 \, {\scriptscriptstyle \pm 5.97}$ &  $32.37 \, {\scriptscriptstyle \pm 14.80}$ &  $74.57 \, {\scriptscriptstyle \pm 4.71}$ &  $38.86 \, {\scriptscriptstyle \pm 1.41}$ \\
GCN-FSF &  $62.00 \, {\scriptscriptstyle \pm 3.41}$ &  $56.34 \, {\scriptscriptstyle \pm 2.88}$ &  $70.03 \, {\scriptscriptstyle \pm 0.94}$ &   $8.18 \, {\scriptscriptstyle \pm 2.53}$ &  $27.66 \, {\scriptscriptstyle \pm 14.66}$ &  $74.50 \, {\scriptscriptstyle \pm 2.85}$ &  $38.95 \, {\scriptscriptstyle \pm 3.38}$ \\
\bottomrule
\end{tabular}
    \label{tab:r_gcn}
\end{table*}

\begin{table*}[t]
        \caption{Results for the Laplacian modifications on \textbf{GAT} for the \textbf{NBA} dataset.}
    \centering
    
\begin{tabular}{lccccccc}
\toprule
 & \multicolumn{7}{c}{Metrics} \\
 \cmidrule(lr){2-8}  & \multicolumn{3}{c}{Utility} & \multicolumn{3}{c}{Fairness} & Trade-off \\
\cmidrule(lr){2-4} \cmidrule(lr){5-7} \cmidrule(lr){8-8} 
  Model &  Acc. ($\uparrow$) &  F1 ($\uparrow$) & AUC ($\uparrow$) & $\Delta_{SP}$ ($\downarrow$) & $\Delta_{EO}$ ($\downarrow$) &  $\text{Con}_{k=10} $ ($\uparrow$) &  DTI ($\downarrow$) \\
\midrule
GAT (Base)     &  $61.50 \, {\scriptscriptstyle \pm 2.42}$ &  $57.64 \, {\scriptscriptstyle \pm 2.50}$ &  $67.68 \, {\scriptscriptstyle \pm 1.75}$ &   $4.31 \, {\scriptscriptstyle \pm 2.53}$ &   $18.64 \, {\scriptscriptstyle \pm 6.67}$ &  $71.88 \, {\scriptscriptstyle \pm 2.34}$ &  $38.83 \, {\scriptscriptstyle \pm 2.28}$ \\
\midrule
GAT-FVP &  $65.75 \, {\scriptscriptstyle \pm 3.22}$ &  $58.51 \, {\scriptscriptstyle \pm 6.93}$ &  $69.84 \, {\scriptscriptstyle \pm 3.80}$ &   $9.58 \, {\scriptscriptstyle \pm 7.08}$ &  $24.37 \, {\scriptscriptstyle \pm 16.06}$ &  $70.37 \, {\scriptscriptstyle \pm 2.23}$ &  $36.07 \, {\scriptscriptstyle \pm 4.96}$ \\
GAT-FSP &  $62.75 \, {\scriptscriptstyle \pm 1.84}$ &  $56.74 \, {\scriptscriptstyle \pm 2.31}$ &  $69.64 \, {\scriptscriptstyle \pm 1.46}$ &  $14.17 \, {\scriptscriptstyle \pm 8.36}$ &  $28.06 \, {\scriptscriptstyle \pm 11.52}$ &  $76.90 \, {\scriptscriptstyle \pm 4.22}$ &  $40.63 \, {\scriptscriptstyle \pm 3.29}$ \\
GAT-FSF &  $65.25 \, {\scriptscriptstyle \pm 4.70}$ &  $61.40 \, {\scriptscriptstyle \pm 5.61}$ &  $71.13 \, {\scriptscriptstyle \pm 3.10}$ &   $4.57 \, {\scriptscriptstyle \pm 3.51}$ &  $24.17 \, {\scriptscriptstyle \pm 12.62}$ &  $74.45 \, {\scriptscriptstyle \pm 1.99}$ &  $35.24 \, {\scriptscriptstyle \pm 4.63}$ \\
\bottomrule
\end{tabular}
    \label{tab:r_gat}
\end{table*}

\begin{table*}[t]
        \caption{Results for the Laplacian modifications on \textbf{NSD} for the \textbf{NBA} dataset (extracted from Table~\ref{tab:r_nba}).}
    \centering
    
\begin{tabular}{lccccccc}
\toprule
 & \multicolumn{7}{c}{Metrics} \\
 \cmidrule(lr){2-8}  & \multicolumn{3}{c}{Utility} & \multicolumn{3}{c}{Fairness} & Trade-off \\
\cmidrule(lr){2-4} \cmidrule(lr){5-7} \cmidrule(lr){8-8} 
  Model &  Acc. ($\uparrow$) &  F1 ($\uparrow$) & AUC ($\uparrow$) & $\Delta_{SP}$ ($\downarrow$) & $\Delta_{EO}$ ($\downarrow$) &  $\text{Con}_{k=10} $ ($\uparrow$) &  DTI ($\downarrow$) \\
\midrule
Diag-NSD & $60.40 \, {\scriptscriptstyle \pm 2.42}$ & $52.38 \, {\scriptscriptstyle \pm 3.38}$ & $70.55 \, {\scriptscriptstyle \pm 2.03}$ & $7.33 \, {\scriptscriptstyle \pm 9.12}$ & $17.79 \, {\scriptscriptstyle \pm 9.84}$ & $79.70 \, {\scriptscriptstyle \pm 5.19}$ & $37.94 \, {\scriptscriptstyle \pm 3.76}$ \\
\midrule
Diag-FVP & $66.34 \, {\scriptscriptstyle \pm 3.23}$ & $64.58 \, {\scriptscriptstyle \pm 9.62}$ & $72.37 \, {\scriptscriptstyle \pm 1.37}$ & $1.69 \, {\scriptscriptstyle \pm 2.57}$ & $7.62 \, {\scriptscriptstyle \pm 46.67}$ & $89.11 \, {\scriptscriptstyle \pm 2.46}$ & $40.00 \, {\scriptscriptstyle \pm 0.00}$ \\
Diag-FSP & $62.38 \, {\scriptscriptstyle \pm 4.34}$ & $50.00 \, {\scriptscriptstyle \pm 17.97}$ & $71.59 \, {\scriptscriptstyle \pm 6.57}$ & $2.72 \, {\scriptscriptstyle \pm 4.18}$ & $20.24 \, {\scriptscriptstyle \pm 17.62}$ & $84.16 \, {\scriptscriptstyle \pm 3.78}$ & $35.03 \, {\scriptscriptstyle \pm 3.79}$ \\
Diag-FSF & $75.25 \, {\scriptscriptstyle \pm 3.26}$ & $65.75 \, {\scriptscriptstyle \pm 3.33}$ & $80.27 \, {\scriptscriptstyle \pm 8.19}$ & $3.90 \, {\scriptscriptstyle \pm 2.44}$ & $10.58 \, {\scriptscriptstyle \pm 40.00}$ & $67.92 \, {\scriptscriptstyle \pm 0.90}$ & $40.00 \, {\scriptscriptstyle \pm 0.00}$ \\
\bottomrule
\end{tabular}
    \label{tab:r_nsd}
\end{table*}

\section{Further experimental results}
\label{app:Results}

\paragraph{Ablation on the NSD and exploration of its generality.} Remember the expression for NSD layers in Eq.~\eqref{NSD}:
\begin{equation*}
    X^{t+1} = X^t - \sigma(\tsheaflap (I_{n} \otimes W_1^t) X^t W_2^t).
\end{equation*}
Setting the stalk dimension to one $d=1$, using the standard graph Laplacian $\Delta_{G}$ (i.e. the trivial sheaf), and dropping the left weights results in:
\begin{equation*}
    X^{t+1} = X^t - \sigma(\Delta_{G} X^t W_2^t),
\end{equation*}
which ammounts to the original GCN with residual connections \cite{GCN}.

On the other hand, setting the stalk dimension to one and dropping the left weights again but keeping the learnable structure on the Laplacian results in this model instead:
\begin{equation*}
    X^{t+1} = X^t - \sigma(\Delta^{\mathcal{F}(t)}_{d=1}  X^t W_2^t),
\end{equation*}
ammounting to the original GAT if using a learnable attention mechanism as transport maps \cite{GAT}. 

Theoretically, GraphSAGE models could be implemented by accounting for the concatenation performed at the end of each layer, for example, by doubling the dimension of the node stalks each iteration. However, this inquiry is beyond the scope of this work. In any case, it is clear that the generality of the sheaf paradigm allows the abstraction of notable GNN architectures. This, along with their ability to handle heterophilic data and oversmoothing problems, motivates their use in fairness applications. Moreover, establishing the theoretical properties of our algorithm for NSD generalizes in a straightforward manner to these architectures.

Finally, we establish this fact empirically on the Pokec-z dataset by implementing the models as previously outlined using the $50\%/25\%/25\%$ cross-validation scheme, resulting in Tables~\ref{tab:r_gcn} and~\ref{tab:r_gat}, which serves as an ablation study on the NSD methodology for the proposed methods. For starters, this small experiment is a success on the GCN, with our models showing superior performance in terms of accuracy and AUC, and fairness in terms of statistical parity. The case for the GAT is more complicated, with the filtering methods achieving improved utility at the cost of fairness. However, looking at the characteristics of the dataset in Table~\ref{tab:datasets} it seems like the GAT is close to predicting the majority class, with our methods trying to break from this pattern and scratching a bit of performance. Moreover, looking at the DTI our solutions are closer to the ideal point, thus showcasing their ability to achieve enhanced results in the Pareto front. If we now check the results shown in Table~\ref{tab:r_nsd} for NSD and Fair NSD with diagonal parametrization in the same dataset we can see how the use of sheaf models results in enhanced fairness and even accuracy, with Diag-FSF achieving top results in both accuracy and fairness among all the considered models. In light of these results, we conjecture that the constraint on the bundle dimension of the GAT configures an obstruction to easily obtaining fair, accurate solutions due to the limitation it imposes on the expressivity of the underlying model, thus motivating the use of NSD in general and our geometric processors in particular.

\paragraph{Dirichlet Energy.} In order to build further intuition, we apply the same models as in the \textbf{Spectrum} experiment and check the Dirichlet energy of both sensitive groups. Figure~\ref{fig:ridge} shows the kernel density plot of the Dirichlet energy of the nodes. As more diffusion layers are added, there is a slight offset between the energy distributions of the privileged and unprivileged groups, whereas in our methods they practically overlap; in other words, our models are incapable of distinguishing between privileged and unprivileged nodes using Dirichlet energy, thus encouraging fair classification.

\paragraph{Sensitivity analysis.} The fairness-utility trade-off for all three methods is mainly governed by a single parameter $\gamma$, granting us fine control to find the desired solution. We study the impact of this central parameter on both accuracy and statistical parity for the Pokec-z dataset, with the result shown in Figure~\ref{fig:tradeoff}. In general, the models behave as expected, with accuracy decreasing and fairness increasing as $\gamma$ is increased. Some upticks in accuracy appear, however, this could be explained by noting that our methods are prone to increasing the spectral gap of the sheaf Laplacian (see Figure~\ref{fig:spectrum}), which is known to mitigate oversquashing and improve the quality of message-passing \cite{FOSR}.

\paragraph{Real-world datasets.} We now display other metrics different than the ones displayed in the aggregated results in Table~\ref{tab:r_extended} and the concrete results for each dataset in Tables~\ref{tab:r_nba} to~\ref{tab:r_compass}. We report the AUC and F1 scores as additional utility metrics, Equal Odds ($\Delta_{EO}$) as another group fairness metrics~\cite{FairMLBook}, and the kNN consistency considering $10$ neighbors and the Euclidean metric ($\text{Con.}_{k=10}$) to quantify individual fairness~\cite{CONSISTENCY}. Finally, the methodology to obtain these results is the same as the one described in Section~\ref{Experiments}, and is described in more detail in Appendix~\ref{app:TrainingDetails}.

Table~\ref{tab:r_extended} shows our methods remain robust when introducing new metrics. For starters, our methods achieve very good results in terms of F1 score, achieving the top three positions while mantaining low statistical parity. On the other hand, the results for AUC, while not as stellar, are still competitive for most of the proposed methods, with Gen-FSF achieving a score of $80.16\%$, just shy of the top, while O(d)-FSF configures one of the worst underperformers in this respect with a score of $66.02\%$. In general, the new utility metrics confirm the picture painted in Section~\ref{Experiments}: our methods are competitive despite sacrificing some performance. On the other hand, the Fair NSD models show poor results in terms of Equal Opportunity: while O(d)-FSP achieves a competitive result of $3.47\%$, most of our methods do not outperform the value for Logistic Regression, with O(d)-FVP achieving the disastrous result of $15.82\%$. However, this is not surprising. After all, our methods aim at achieving Statistical Parity, and it is a well-known fact that these two fairness metrics are independent when not incompatible~\cite{INCOMPATIBLE}. Moreover, our methodology could be extended to handle Equal Odds as suggested in Section~\ref{Theory}, although we leave such inquiries for future work. Finally, the introduction of an individual fairness metric shows our methods also achieve good results, with O(d)-FSP achieving the best value for consistency ($85.2\%$) and the rest of our processors achieve values around $80\%$, remaining competitive with respect to the rest of the algorithms and outperforming, most notably, tabular models. However, this outcome is hardly surprising due to the relationship between individual fairness and the graph heat equation \cite{FairSheafDiffusion}.

The situation for each individual dataset is detailed in Tables~\ref{tab:r_nba} to~\ref{tab:r_compass}, which overall repeat the same patterns as the aggregated results. One noteworthy example of the success of our methods is found in the relatively hard NBA dataset, in which our models achieve top performance in both Statistical Parity and utility (with the notable exception of O(d)-FSP), showing how they can remain competitive with the landscape of fair Machine Learning and even achieve new, better results. Its also important to note the case of the German dataset, in which our models struggle to reduce bias without collapsing to the mode. However, note that the rest of the fair architectures also struggle with this dataset, which serves as a sanity check and demonstrates this phenomenom is inherent to the dataset itself, perhaps due to the synthetic nature of the graph. Furthermore, the O(d) parametrization with FVP modifications exhibits severe performance degradation, although less so than other failing architectures like GAT. Moreover, note that the general architectures using FSF filtering manage to slightly break from this pattern, scratching some extra utility while sacrificing minimal fairness, thus achieving a different trade-off. Nonetheless, this is an outlier outshone by the rest of the datasets.

Finally, it is important to note that some of the worst results observed for certain baselines can be attributed to the training procedure. While these models may achieve higher accuracy when fairness is unconstrained, they fail to maintain performance once we enforce Statistical Parity with a tolerance of $20\%$ in the grid search. This suggests a limited ability to navigate the Pareto frontier between utility and fairness, and underscores the contribution of our models in this goal.

\begin{table*}[t]
\centering
\caption{\textbf{Hyperparameter search space for the empirical assesment.}\\ Detailed list of the grids and distributions used to search for the optimal parameters for all models. Unless stated otherwise, assume the uniform distribution on each set or interval.}
\label{tab:grid_search}
\begin{tabular}{ll l}
\toprule
\textbf{Model} & \textbf{Hyperparameter} & \textbf{Search space} \\ 
\midrule

\multicolumn{3}{c}{\textbf{Common parameters to all models}} \\ 
\cmidrule(lr){1-3}
\multirow{4}{*}{Global} 
 & Optimizer & Adam (${\beta}_1=0.9, {\beta}_2=0.999$) \\
 & Epochs & $100$ (Search), $400$ (Final) \\
 & Learning Rate (lr) & $[10^{-4}, 5\cdot10^{-2}]$ (log-uniform) \\
 & Weight Decay & $[10^{-5}, 10^{-2}]$ (log-uniform) \\ 
\midrule

\multicolumn{3}{c}{\textbf{Neural Sheaf Diffusion}} \\ 
\cmidrule(lr){1-3}
\multirow{9}{*}{\shortstack[l]{NSD Family\\(Diag, O(d), Gen)}} 
 & Layers & $\{1, 2, 3, 4, 5, 6\}$ \\
\bottomrule
\end{tabular}%
\end{table*}

\begin{table*}[t]
\centering
\caption{\textbf{Extended aggregated results.} Our methods achieve top results in F1 score and statistical parity, and remain competitive in accuracy, AUC, consistency and DTI. The only models competing with our methods are fair tabular models, although our algorithms are generally more stable with smaller error bars.}
\resizebox{\textwidth}{!}{
\begin{tabular}{llccccccc}
\toprule
& {} & \multicolumn{7}{c}{Metrics} \\
 \cmidrule(lr){3-9} & & \multicolumn{3}{c}{Utility} & \multicolumn{3}{c}{Fairness} & Trade-off \\
\cmidrule(lr){3-5} \cmidrule(lr){6-8} \cmidrule(lr){9-9} 
Familiy & Model &  Acc. ($\uparrow$) &  F1 ($\uparrow$) & AUC ($\uparrow$) & $\Delta_{SP}$ ($\downarrow$) & $\Delta_{EO}$ ($\downarrow$) &  $\text{Con}_{k=10} $ ($\uparrow$) &  DTI ($\downarrow$) \\
\midrule
\multirow{9}{*}{Fair NSD (Ours)} & Diag-FVP & $66.91 \, {\scriptscriptstyle \pm 4.92}$ & $62.65 \, {\scriptscriptstyle \pm 17.71}$ & $67.10 \, {\scriptscriptstyle \pm 11.44}$ & \third{$5.10 \, {\scriptscriptstyle \pm 8.83}$} & $6.09 \, {\scriptscriptstyle \pm 2.63}$ & $82.49 \, {\scriptscriptstyle \pm 18.08}$ & $34.10 \, {\scriptscriptstyle \pm 7.04}$ \\
 & O(d)-FVP & $59.40 \, {\scriptscriptstyle \pm 16.43}$ & $51.32 \, {\scriptscriptstyle \pm 22.37}$ & $68.23 \, {\scriptscriptstyle \pm 12.21}$ & $10.99 \, {\scriptscriptstyle \pm 9.25}$ & $17.02 \, {\scriptscriptstyle \pm 26.06}$ & $80.80 \, {\scriptscriptstyle \pm 15.21}$ & $43.16 \, {\scriptscriptstyle \pm 15.43}$ \\
 & Gen-FVP & $67.75 \, {\scriptscriptstyle \pm 4.76}$ & \first{$67.38 \, {\scriptscriptstyle \pm 9.69}$} & $68.47 \, {\scriptscriptstyle \pm 10.18}$ & $17.10 \, {\scriptscriptstyle \pm 21.58}$ & $8.67 \, {\scriptscriptstyle \pm 11.54}$ & $79.86 \, {\scriptscriptstyle \pm 16.13}$ & $39.34 \, {\scriptscriptstyle \pm 14.82}$ \\
 & Diag-FSP & $68.20 \, {\scriptscriptstyle \pm 3.81}$ & $65.03 \, {\scriptscriptstyle \pm 12.18}$ & $69.29 \, {\scriptscriptstyle \pm 10.93}$ & $7.44 \, {\scriptscriptstyle \pm 10.55}$ & $8.15 \, {\scriptscriptstyle \pm 7.00}$ & $82.87 \, {\scriptscriptstyle \pm 15.68}$ & $33.74 \, {\scriptscriptstyle \pm 6.04}$ \\
\bottomrule
\end{tabular}}
\label{tab:r_extended}
\end{table*}

\begin{table*}[t]
\centering
\caption{Extended results for the \textbf{NBA} dataset. Most of the graph models collapse to the mode, highlighting the difficulty in dealing with heterophily in both the labels and sensitive attribute. Our methods remain competitive and even achieve the top scores in accuracy.}
\resizebox{\textwidth}{!}{
\begin{tabular}{llccccccc}
\toprule
& {} & \multicolumn{7}{c}{Metrics} \\
 \cmidrule(lr){3-9} & & \multicolumn{3}{c}{Utility} & \multicolumn{3}{c}{Fairness} & Trade-off \\
\cmidrule(lr){3-5} \cmidrule(lr){6-8} \cmidrule(lr){9-9} 
Familiy & Model &  Acc. ($\uparrow$) &  F1 ($\uparrow$) & AUC ($\uparrow$) & $\Delta_{SP}$ ($\downarrow$) & $\Delta_{EO}$ ($\downarrow$) &  $\text{Con}_{k=10} $ ($\uparrow$) &  DTI ($\downarrow$) \\
\midrule
\multirow{9}{*}{Fair NSD (Ours)} & Diag-FVP & $66.34 \, {\scriptscriptstyle \pm 3.23}$ & \second{$64.58 \, {\scriptscriptstyle \pm 9.62}$} & $72.37 \, {\scriptscriptstyle \pm 1.37}$ & $1.69 \, {\scriptscriptstyle \pm 2.57}$ & $7.62 \, {\scriptscriptstyle \pm 46.67}$ & \third{$89.11 \, {\scriptscriptstyle \pm 2.46}$} & $40.00 \, {\scriptscriptstyle \pm 0.00}$ \\
 & O(d)-FVP & $62.38 \, {\scriptscriptstyle \pm 6.92}$ & $54.76 \, {\scriptscriptstyle \pm 23.60}$ & $72.00 \, {\scriptscriptstyle \pm 14.54}$ & $2.15 \, {\scriptscriptstyle \pm 6.96}$ & $9.85 \, {\scriptscriptstyle \pm 36.11}$ & $85.15 \, {\scriptscriptstyle \pm 7.36}$ & $38.57 \, {\scriptscriptstyle \pm 1.91}$ \\
 & Gen-FVP & $60.40 \, {\scriptscriptstyle \pm 6.53}$ & \third{$62.26 \, {\scriptscriptstyle \pm 7.02}$} & $68.03 \, {\scriptscriptstyle \pm 8.46}$ & $47.90 \, {\scriptscriptstyle \pm 33.90}$ & $29.24 \, {\scriptscriptstyle \pm 33.71}$ & $72.97 \, {\scriptscriptstyle \pm 15.67}$ & $60.12 \, {\scriptscriptstyle \pm 32.23}$ \\
 & Diag-FSP & $62.38 \, {\scriptscriptstyle \pm 4.34}$ & $50.00 \, {\scriptscriptstyle \pm 17.97}$ & $71.59 \, {\scriptscriptstyle \pm 6.57}$ & $2.72 \, {\scriptscriptstyle \pm 4.18}$ & $20.24 \, {\scriptscriptstyle \pm 17.62}$ & $84.16 \, {\scriptscriptstyle \pm 3.78}$ & $35.03 \, {\scriptscriptstyle \pm 3.79}$ \\
\bottomrule
\end{tabular}}
\label{tab:r_nba}
\end{table*}

\begin{table*}[t]
\centering
\caption{Extended results for the \textbf{Pokec-n} dataset. In this dataset our methods achieve top positions for all metrics except the utility-fairness trade-off.}
\resizebox{\textwidth}{!}{
\begin{tabular}{llccccccc}
\toprule
& {} & \multicolumn{7}{c}{Metrics} \\
 \cmidrule(lr){3-9} & & \multicolumn{3}{c}{Utility} & \multicolumn{3}{c}{Fairness} & Trade-off \\
\cmidrule(lr){3-5} \cmidrule(lr){6-8} \cmidrule(lr){9-9} 
Familiy & Model &  Acc. ($\uparrow$) &  F1 ($\uparrow$) & AUC ($\uparrow$) & $\Delta_{SP}$ ($\downarrow$) & $\Delta_{EO}$ ($\downarrow$) &  $\text{Con}_{k=10} $ ($\uparrow$) &  DTI ($\downarrow$) \\
\midrule
\multirow{9}{*}{Fair NSD (Ours)} & Diag-FVP & $70.01 \, {\scriptscriptstyle \pm 3.05}$ & \second{$65.06 \, {\scriptscriptstyle \pm 5.10}$} & $76.26 \, {\scriptscriptstyle \pm 2.59}$ & \second{$0.38 \, {\scriptscriptstyle \pm 10.64}$} & $6.22 \, {\scriptscriptstyle \pm 10.75}$ & $61.45 \, {\scriptscriptstyle \pm 2.35}$ & $37.71 \, {\scriptscriptstyle \pm 2.66}$ \\
 & O(d)-FVP & $69.04 \, {\scriptscriptstyle \pm 1.64}$ & $60.51 \, {\scriptscriptstyle \pm 2.68}$ & $74.14 \, {\scriptscriptstyle \pm 1.83}$ & $15.27 \, {\scriptscriptstyle \pm 7.14}$ & $2.54 \, {\scriptscriptstyle \pm 4.84}$ & $65.69 \, {\scriptscriptstyle \pm 4.16}$ & $57.39 \, {\scriptscriptstyle \pm 3.79}$ \\
 & Gen-FVP & $69.68 \, {\scriptscriptstyle \pm 2.28}$ & $59.95 \, {\scriptscriptstyle \pm 3.70}$ & \third{$76.66 \, {\scriptscriptstyle \pm 1.10}$} & $4.11 \, {\scriptscriptstyle \pm 2.31}$ & $2.62 \, {\scriptscriptstyle \pm 5.23}$ & $67.94 \, {\scriptscriptstyle \pm 4.16}$ & $37.39 \, {\scriptscriptstyle \pm 4.07}$ \\
 & Diag-FSP & \first{$71.36 \, {\scriptscriptstyle \pm 3.36}$} & $61.58 \, {\scriptscriptstyle \pm 2.23}$ & $76.45 \, {\scriptscriptstyle \pm 1.43}$ & $6.14 \, {\scriptscriptstyle \pm 3.54}$ & $2.42 \, {\scriptscriptstyle \pm 3.22}$ & $70.39 \, {\scriptscriptstyle \pm 3.41}$ & $38.77 \, {\scriptscriptstyle \pm 3.25}$ \\
\bottomrule
\end{tabular}}
\label{tab:r_pokec_n}
\end{table*}

\begin{table*}[t]
\centering
\caption{Extended results for the \textbf{Pokec-z} dataset. Our methods achieve some of the top spots in F1 score, statistical parity and consistency. Furthermore, they achieve good results in accuracy and AUC and remain competitive in equal odds and DTI.}
\resizebox{\textwidth}{!}{
\begin{tabular}{llccccccc}
\toprule
& {} & \multicolumn{7}{c}{Metrics} \\
 \cmidrule(lr){3-9} & & \multicolumn{3}{c}{Utility} & \multicolumn{3}{c}{Fairness} & Trade-off \\
\cmidrule(lr){3-5} \cmidrule(lr){6-8} \cmidrule(lr){9-9} 
Familiy & Model &  Acc. ($\uparrow$) &  F1 ($\uparrow$) & AUC ($\uparrow$) & $\Delta_{SP}$ ($\downarrow$) & $\Delta_{EO}$ ($\downarrow$) &  $\text{Con}_{k=10} $ ($\uparrow$) &  DTI ($\downarrow$) \\
\midrule
\multirow{9}{*}{Fair NSD (Ours)} & Diag-FVP & $69.24 \, {\scriptscriptstyle \pm 1.46}$ & $67.90 \, {\scriptscriptstyle \pm 1.94}$ & $76.78 \, {\scriptscriptstyle \pm 0.63}$ & $0.75 \, {\scriptscriptstyle \pm 2.88}$ & $3.44 \, {\scriptscriptstyle \pm 2.05}$ & $64.92 \, {\scriptscriptstyle \pm 1.81}$ & $34.79 \, {\scriptscriptstyle \pm 0.98}$ \\
 & O(d)-FVP & $68.77 \, {\scriptscriptstyle \pm 0.66}$ & $67.39 \, {\scriptscriptstyle \pm 1.23}$ & $76.88 \, {\scriptscriptstyle \pm 0.47}$ & $5.75 \, {\scriptscriptstyle \pm 2.22}$ & $1.14 \, {\scriptscriptstyle \pm 1.81}$ & $63.81 \, {\scriptscriptstyle \pm 2.35}$ & $33.61 \, {\scriptscriptstyle \pm 0.63}$ \\
 & Gen-FVP & $69.82 \, {\scriptscriptstyle \pm 0.62}$ & $68.62 \, {\scriptscriptstyle \pm 0.70}$ & $76.26 \, {\scriptscriptstyle \pm 0.95}$ & \third{$0.62 \, {\scriptscriptstyle \pm 3.67}$} & $2.47 \, {\scriptscriptstyle \pm 2.61}$ & $64.04 \, {\scriptscriptstyle \pm 0.88}$ & $33.93 \, {\scriptscriptstyle \pm 0.73}$ \\
 & Diag-FSP & $69.61 \, {\scriptscriptstyle \pm 0.80}$ & \second{$68.81 \, {\scriptscriptstyle \pm 1.27}$} & $76.68 \, {\scriptscriptstyle \pm 1.29}$ & $2.42 \, {\scriptscriptstyle \pm 3.86}$ & $4.51 \, {\scriptscriptstyle \pm 2.21}$ & $63.87 \, {\scriptscriptstyle \pm 0.68}$ & $34.64 \, {\scriptscriptstyle \pm 0.86}$ \\
\bottomrule
\end{tabular}}
\label{tab:r_pokec_z}
\end{table*}

\begin{table*}[t]
\centering
\caption{Extended results for the \textbf{German} dataset. Most models collapse to the mode. This time our methods suffer the same fate, suggesting a obstruction to fairness inherent to this dataset.}
\resizebox{\textwidth}{!}{
\begin{tabular}{llccccccc}
\toprule
& {} & \multicolumn{7}{c}{Metrics} \\
 \cmidrule(lr){3-9} & & \multicolumn{3}{c}{Utility} & \multicolumn{3}{c}{Fairness} & Trade-off \\
\cmidrule(lr){3-5} \cmidrule(lr){6-8} \cmidrule(lr){9-9} 
Familiy & Model &  Acc. ($\uparrow$) &  F1 ($\uparrow$) & AUC ($\uparrow$) & $\Delta_{SP}$ ($\downarrow$) & $\Delta_{EO}$ ($\downarrow$) &  $\text{Con}_{k=10} $ ($\uparrow$) &  DTI ($\downarrow$) \\
\midrule
\multirow{8}{*}{Fair NSD (Ours)} & Diag-FVP & $70.40 \, {\scriptscriptstyle \pm 2.07}$ & $82.04 \, {\scriptscriptstyle \pm 0.00}$ & $51.37 \, {\scriptscriptstyle \pm 0.00}$ & $1.82 \, {\scriptscriptstyle \pm 0.00}$ & $3.59 \, {\scriptscriptstyle \pm 0.00}$ & $99.08 \, {\scriptscriptstyle \pm 0.00}$ & $70.50 \, {\scriptscriptstyle \pm 0.00}$ \\
 & O(d)-FVP & $30.40 \, {\scriptscriptstyle \pm 3.39}$ & $12.12 \, {\scriptscriptstyle \pm 10.17}$ & $46.74 \, {\scriptscriptstyle \pm 1.79}$ & $6.65 \, {\scriptscriptstyle \pm 0.96}$ & $63.16 \, {\scriptscriptstyle \pm 20.91}$ & $93.20 \, {\scriptscriptstyle \pm 11.05}$ & $68.97 \, {\scriptscriptstyle \pm 5.90}$ \\
 & Gen-FVP & $72.80 \, {\scriptscriptstyle \pm 2.07}$ & $83.73 \, {\scriptscriptstyle \pm 0.00}$ & $51.59 \, {\scriptscriptstyle \pm 0.37}$ & $1.05 \, {\scriptscriptstyle \pm 0.00}$ & $4.40 \, {\scriptscriptstyle \pm 0.00}$ & $98.24 \, {\scriptscriptstyle \pm 0.00}$ & $70.50 \, {\scriptscriptstyle \pm 0.00}$ \\
 & Diag-FSP & $71.20 \, {\scriptscriptstyle \pm 19.56}$ & $83.18 \, {\scriptscriptstyle \pm 40.12}$ & $50.26 \, {\scriptscriptstyle \pm 1.27}$ & \first{$0.00 \, {\scriptscriptstyle \pm 4.16}$} & $6.40 \, {\scriptscriptstyle \pm 10.70}$ & \first{$100.00 \, {\scriptscriptstyle \pm 0.22}$} & $45.70 \, {\scriptscriptstyle \pm 20.25}$ \\
\bottomrule
\end{tabular}}
\label{tab:r_german}
\end{table*}

\begin{table*}[t]
\centering
\caption{Extended results for the \textbf{Compas} dataset. Our models achieve competitive performance while greatly reducing the unfairness of the baselines, both in the case of graph models and tabular algorithms. Our models mostly dominate all benchmarks with the exception of fair tabular models.}
\resizebox{\textwidth}{!}{
\begin{tabular}{llccccccc}
\toprule
& {} & \multicolumn{7}{c}{Metrics} \\
 \cmidrule(lr){3-9} & & \multicolumn{3}{c}{Utility} & \multicolumn{3}{c}{Fairness} & Trade-off \\
\cmidrule(lr){3-5} \cmidrule(lr){6-8} \cmidrule(lr){9-9} 
Familiy & Model &  Acc. ($\uparrow$) &  F1 ($\uparrow$) & AUC ($\uparrow$) & $\Delta_{SP}$ ($\downarrow$) & $\Delta_{EO}$ ($\downarrow$) &  $\text{Con}_{k=10} $ ($\uparrow$) &  DTI ($\downarrow$) \\
\midrule
\multirow{9}{*}{Fair NSD (Ours)} & Diag-FVP & $58.59 \, {\scriptscriptstyle \pm 3.50}$ & $33.64 \, {\scriptscriptstyle \pm 13.92}$ & $58.70 \, {\scriptscriptstyle \pm 8.58}$ & $20.86 \, {\scriptscriptstyle \pm 10.02}$ & $9.56 \, {\scriptscriptstyle \pm 20.43}$ & \third{$97.89 \, {\scriptscriptstyle \pm 0.98}$} & $48.76 \, {\scriptscriptstyle \pm 6.74}$ \\
 & O(d)-FVP & $66.43 \, {\scriptscriptstyle \pm 2.95}$ & $61.80 \, {\scriptscriptstyle \pm 2.91}$ & $71.41 \, {\scriptscriptstyle \pm 2.60}$ & $25.12 \, {\scriptscriptstyle \pm 16.83}$ & $8.43 \, {\scriptscriptstyle \pm 4.85}$ & $96.15 \, {\scriptscriptstyle \pm 0.31}$ & $56.43 \, {\scriptscriptstyle \pm 14.25}$ \\
 & Gen-FVP & $66.04 \, {\scriptscriptstyle \pm 3.96}$ & $62.36 \, {\scriptscriptstyle \pm 2.35}$ & $69.82 \, {\scriptscriptstyle \pm 3.34}$ & $31.80 \, {\scriptscriptstyle \pm 20.59}$ & $4.65 \, {\scriptscriptstyle \pm 1.89}$ & $96.09 \, {\scriptscriptstyle \pm 2.09}$ & $63.60 \, {\scriptscriptstyle \pm 18.41}$ \\
 & Diag-FSP & $66.43 \, {\scriptscriptstyle \pm 0.90}$ & $61.57 \, {\scriptscriptstyle \pm 1.92}$ & $71.49 \, {\scriptscriptstyle \pm 1.97}$ & $25.90 \, {\scriptscriptstyle \pm 5.81}$ & $7.19 \, {\scriptscriptstyle \pm 5.17}$ & $95.96 \, {\scriptscriptstyle \pm 0.65}$ & $39.87 \, {\scriptscriptstyle \pm 2.29}$ \\
\bottomrule
\end{tabular}}
\label{tab:r_compass}
\end{table*}

\begin{table*}[t]
\centering
\caption{Extended results for the \textbf{Nba} dataset.}
\resizebox{\textwidth}{!}{
\begin{tabular}{llccccccccc}
\toprule
& {} & \multicolumn{9}{c}{Metrics} \\
 \cmidrule(lr){3-11} & & \multicolumn{3}{c}{Utility} & \multicolumn{3}{c}{Fairness} & Trade-off & \multicolumn{2}{c}{Efficiency} \\
\midrule
\multirow{9}{*}{Fair NSD (Ours)} & Diag-FVP & $70.71 \, {\scriptscriptstyle \pm 13.80}$ & $47.93 \, {\scriptscriptstyle \pm 25.68}$ & $57.27 \, {\scriptscriptstyle \pm 25.69}$ & $6.17 \, {\scriptscriptstyle \pm 5.23}$ & $7.21 \, {\scriptscriptstyle \pm 6.93}$ & \third{$85.76 \, {\scriptscriptstyle \pm 13.68}$} & $30.13 \, {\scriptscriptstyle \pm 14.29}$ & $2.186 \, {\scriptscriptstyle \pm 2.105}$ & $0.003 \, {\scriptscriptstyle \pm 0.002}$ \\
 & O(d)-FVP & $72.10 \, {\scriptscriptstyle \pm 12.72}$ & $55.72 \, {\scriptscriptstyle \pm 25.93}$ & $58.99 \, {\scriptscriptstyle \pm 27.11}$ & $11.99 \, {\scriptscriptstyle \pm 11.61}$ & $4.46 \, {\scriptscriptstyle \pm 5.89}$ & $81.24 \, {\scriptscriptstyle \pm 14.43}$ & $31.14 \, {\scriptscriptstyle \pm 15.54}$ & $75.581 \, {\scriptscriptstyle \pm 193.435}$ & $0.146 \, {\scriptscriptstyle \pm 0.374}$ \\
 & Gen-FVP & $69.22 \, {\scriptscriptstyle \pm 14.34}$ & $47.61 \, {\scriptscriptstyle \pm 27.90}$ & $56.11 \, {\scriptscriptstyle \pm 25.70}$ & $5.11 \, {\scriptscriptstyle \pm 6.35}$ & $5.68 \, {\scriptscriptstyle \pm 5.70}$ & $85.54 \, {\scriptscriptstyle \pm 13.29}$ & $31.54 \, {\scriptscriptstyle \pm 14.87}$ & $3.241 \, {\scriptscriptstyle \pm 2.754}$ & $0.008 \, {\scriptscriptstyle \pm 0.008}$ \\
 & Diag-FSP & $68.75 \, {\scriptscriptstyle \pm 14.47}$ & $49.05 \, {\scriptscriptstyle \pm 26.28}$ & $55.99 \, {\scriptscriptstyle \pm 25.56}$ & $5.22 \, {\scriptscriptstyle \pm 8.21}$ & $4.09 \, {\scriptscriptstyle \pm 4.00}$ & $84.43 \, {\scriptscriptstyle \pm 14.86}$ & $32.22 \, {\scriptscriptstyle \pm 15.37}$ & $15.641 \, {\scriptscriptstyle \pm 20.559}$ & $0.002 \, {\scriptscriptstyle \pm 0.002}$ \\
\bottomrule
\end{tabular}}
\label{tab:extended_results}
\end{table*}

\begin{table*}[t]
\centering
\caption{Extended results for the \textbf{Pokec n} dataset.}
\resizebox{\textwidth}{!}{
\begin{tabular}{llccccccccc}
\toprule
& {} & \multicolumn{9}{c}{Metrics} \\
 \cmidrule(lr){3-11} & & \multicolumn{3}{c}{Utility} & \multicolumn{3}{c}{Fairness} & Trade-off & \multicolumn{2}{c}{Efficiency} \\
\cmidrule(lr){3-5} \cmidrule(lr){6-8} \cmidrule(lr){9-9} \cmidrule(lr){10-11}
Familiy & Model &  Acc. ($\uparrow$) &  F1 ($\uparrow$) & AUC ($\uparrow$) & $\Delta_{SP}$ ($\downarrow$) & $\Delta_{EO}$ ($\downarrow$) &  $\text{Con}_{k=10} $ ($\uparrow$) &  DTI ($\downarrow$) & Train T. (s) ($\downarrow$) & Inf. T. (s) ($\downarrow$) \\
\midrule
\multirow{9}{*}{Fair NSD (Ours)} & Diag-FVP & \second{$70.52 \, {\scriptscriptstyle \pm 3.04}$} & \second{$65.61 \, {\scriptscriptstyle \pm 5.44}$} & \second{$76.38 \, {\scriptscriptstyle \pm 2.73}$} & \first{$0.01 \, {\scriptscriptstyle \pm 9.85}$} & $5.49 \, {\scriptscriptstyle \pm 11.63}$ & $61.69 \, {\scriptscriptstyle \pm 3.05}$ & $37.30 \, {\scriptscriptstyle \pm 2.51}$ & $4.416 \, {\scriptscriptstyle \pm 1.269}$ & $0.006 \, {\scriptscriptstyle \pm 0.000}$ \\
 & O(d)-FVP & $69.04 \, {\scriptscriptstyle \pm 1.64}$ & $60.51 \, {\scriptscriptstyle \pm 2.68}$ & $74.14 \, {\scriptscriptstyle \pm 1.83}$ & $15.27 \, {\scriptscriptstyle \pm 7.14}$ & $2.54 \, {\scriptscriptstyle \pm 4.84}$ & $65.69 \, {\scriptscriptstyle \pm 4.16}$ & $57.39 \, {\scriptscriptstyle \pm 3.79}$ & $6.663 \, {\scriptscriptstyle \pm 1.315}$ & $0.012 \, {\scriptscriptstyle \pm 0.000}$ \\
 & Gen-FVP & $69.23 \, {\scriptscriptstyle \pm 2.39}$ & $60.40 \, {\scriptscriptstyle \pm 2.66}$ & $75.11 \, {\scriptscriptstyle \pm 1.86}$ & $8.15 \, {\scriptscriptstyle \pm 3.23}$ & $0.61 \, {\scriptscriptstyle \pm 3.76}$ & $68.19 \, {\scriptscriptstyle \pm 3.85}$ & $36.06 \, {\scriptscriptstyle \pm 2.72}$ & $3.200 \, {\scriptscriptstyle \pm 0.263}$ & $0.003 \, {\scriptscriptstyle \pm 0.000}$ \\
\bottomrule
\end{tabular}}
\label{tab:r_pokec_n}
\end{table*}

\begin{table*}[t]
\centering
\caption{Extended results for the \textbf{Pokec z} dataset.}
\resizebox{\textwidth}{!}{
\begin{tabular}{llccccccccc}
\toprule
& {} & \multicolumn{9}{c}{Metrics} \\
 \cmidrule(lr){3-11} & & \multicolumn{3}{c}{Utility} & \multicolumn{3}{c}{Fairness} & Trade-off & \multicolumn{2}{c}{Efficiency} \\
\cmidrule(lr){3-5} \cmidrule(lr){6-8} \cmidrule(lr){9-9} \cmidrule(lr){10-11}
Familiy & Model &  Acc. ($\uparrow$) &  F1 ($\uparrow$) & AUC ($\uparrow$) & $\Delta_{SP}$ ($\downarrow$) & $\Delta_{EO}$ ($\downarrow$) &  $\text{Con}_{k=10} $ ($\uparrow$) &  DTI ($\downarrow$) & Train T. (s) ($\downarrow$) & Inf. T. (s) ($\downarrow$) \\
\midrule
\multirow{8}{*}{Fair NSD (Ours)} & Diag-FVP & \third{$69.71 \, {\scriptscriptstyle \pm 1.24}$} & $68.38 \, {\scriptscriptstyle \pm 1.36}$ & $76.81 \, {\scriptscriptstyle \pm 0.74}$ & $3.52 \, {\scriptscriptstyle \pm 3.34}$ & $1.16 \, {\scriptscriptstyle \pm 2.46}$ & $63.50 \, {\scriptscriptstyle \pm 1.53}$ & $34.36 \, {\scriptscriptstyle \pm 1.16}$ & $1.678 \, {\scriptscriptstyle \pm 0.296}$ & $0.002 \, {\scriptscriptstyle \pm 0.000}$ \\
 & O(d)-FVP & $68.77 \, {\scriptscriptstyle \pm 0.66}$ & $67.39 \, {\scriptscriptstyle \pm 1.23}$ & \third{$76.88 \, {\scriptscriptstyle \pm 0.47}$} & $5.75 \, {\scriptscriptstyle \pm 2.22}$ & $1.14 \, {\scriptscriptstyle \pm 1.81}$ & $63.81 \, {\scriptscriptstyle \pm 2.35}$ & $33.61 \, {\scriptscriptstyle \pm 0.63}$ & $1.872 \, {\scriptscriptstyle \pm 0.593}$ & $0.004 \, {\scriptscriptstyle \pm 0.000}$ \\
 & Diag-FSP & $68.51 \, {\scriptscriptstyle \pm 1.11}$ & $66.89 \, {\scriptscriptstyle \pm 1.19}$ & $76.40 \, {\scriptscriptstyle \pm 1.60}$ & \first{$0.35 \, {\scriptscriptstyle \pm 3.37}$} & $3.77 \, {\scriptscriptstyle \pm 2.33}$ & $64.66 \, {\scriptscriptstyle \pm 0.38}$ & $34.28 \, {\scriptscriptstyle \pm 0.72}$ & $5.490 \, {\scriptscriptstyle \pm 0.557}$ & $0.004 \, {\scriptscriptstyle \pm 0.000}$ \\
\bottomrule
\end{tabular}}
\label{tab:r_pokec_z}
\end{table*}

\begin{table*}[t]
\centering
\caption{Extended results for the \textbf{German} dataset.}
\resizebox{\textwidth}{!}{
\begin{tabular}{llccccccccc}
\toprule
& {} & \multicolumn{9}{c}{Metrics} \\
 \cmidrule(lr){3-11} & & \multicolumn{3}{c}{Utility} & \multicolumn{3}{c}{Fairness} & Trade-off & \multicolumn{2}{c}{Efficiency} \\
\cmidrule(lr){3-5} \cmidrule(lr){6-8} \cmidrule(lr){9-9} \cmidrule(lr){10-11}
Familiy & Model &  Acc. ($\uparrow$) &  F1 ($\uparrow$) & AUC ($\uparrow$) & $\Delta_{SP}$ ($\downarrow$) & $\Delta_{EO}$ ($\downarrow$) &  $\text{Con}_{k=10} $ ($\uparrow$) &  DTI ($\downarrow$) & Train T. (s) ($\downarrow$) & Inf. T. (s) ($\downarrow$) \\
\midrule
\multirow{8}{*}{Fair NSD (Ours)} & Diag-FVP & $70.40 \, {\scriptscriptstyle \pm 2.07}$ & $82.04 \, {\scriptscriptstyle \pm 0.00}$ & $51.37 \, {\scriptscriptstyle \pm 0.00}$ & $1.82 \, {\scriptscriptstyle \pm 0.00}$ & $3.59 \, {\scriptscriptstyle \pm 0.00}$ & $99.08 \, {\scriptscriptstyle \pm 0.00}$ & $70.50 \, {\scriptscriptstyle \pm 0.00}$ & $0.773 \, {\scriptscriptstyle \pm 0.316}$ & $0.002 \, {\scriptscriptstyle \pm 0.000}$ \\
 & O(d)-FVP & $71.20 \, {\scriptscriptstyle \pm 2.07}$ & $83.18 \, {\scriptscriptstyle \pm 0.00}$ & $30.83 \, {\scriptscriptstyle \pm 0.18}$ & \first{$0.00 \, {\scriptscriptstyle \pm 0.00}$} & $6.40 \, {\scriptscriptstyle \pm 0.00}$ & \first{$100.00 \, {\scriptscriptstyle \pm 0.15}$} & $69.03 \, {\scriptscriptstyle \pm 2.94}$ & $2.698 \, {\scriptscriptstyle \pm 1.610}$ & $0.431 \, {\scriptscriptstyle \pm 0.196}$ \\
 & Gen-FVP & $71.60 \, {\scriptscriptstyle \pm 2.07}$ & $83.14 \, {\scriptscriptstyle \pm 0.00}$ & $53.20 \, {\scriptscriptstyle \pm 0.00}$ & \second{$0.43 \, {\scriptscriptstyle \pm 0.00}$} & $5.41 \, {\scriptscriptstyle \pm 0.00}$ & \third{$99.72 \, {\scriptscriptstyle \pm 0.00}$} & $70.50 \, {\scriptscriptstyle \pm 0.00}$ & $1.677 \, {\scriptscriptstyle \pm 0.278}$ & $0.003 \, {\scriptscriptstyle \pm 0.000}$ \\
\bottomrule
\end{tabular}}
\label{tab:r_german}
\end{table*}

\begin{table*}[t]
\centering
\caption{Extended results for the \textbf{Compass} dataset.}
\resizebox{\textwidth}{!}{
\begin{tabular}{llccccccccc}
\toprule
& {} & \multicolumn{9}{c}{Metrics} \\
 \cmidrule(lr){3-11} & & \multicolumn{3}{c}{Utility} & \multicolumn{3}{c}{Fairness} & Trade-off & \multicolumn{2}{c}{Efficiency} \\
\cmidrule(lr){3-5} \cmidrule(lr){6-8} \cmidrule(lr){9-9} \cmidrule(lr){10-11}
Familiy & Model &  Acc. ($\uparrow$) &  F1 ($\uparrow$) & AUC ($\uparrow$) & $\Delta_{SP}$ ($\downarrow$) & $\Delta_{EO}$ ($\downarrow$) &  $\text{Con}_{k=10} $ ($\uparrow$) &  DTI ($\downarrow$) & Train T. (s) ($\downarrow$) & Inf. T. (s) ($\downarrow$) \\
\midrule
\multirow{9}{*}{Fair NSD (Ours)} & Diag-FVP & $58.59 \, {\scriptscriptstyle \pm 3.48}$ & $33.64 \, {\scriptscriptstyle \pm 13.91}$ & $58.70 \, {\scriptscriptstyle \pm 8.58}$ & $20.86 \, {\scriptscriptstyle \pm 10.06}$ & $9.56 \, {\scriptscriptstyle \pm 20.46}$ & $97.89 \, {\scriptscriptstyle \pm 0.98}$ & $48.77 \, {\scriptscriptstyle \pm 6.75}$ & $2.946 \, {\scriptscriptstyle \pm 0.798}$ & $0.007 \, {\scriptscriptstyle \pm 0.000}$ \\
 & O(d)-FVP & $66.43 \, {\scriptscriptstyle \pm 2.95}$ & $61.80 \, {\scriptscriptstyle \pm 2.91}$ & $71.41 \, {\scriptscriptstyle \pm 2.60}$ & $25.12 \, {\scriptscriptstyle \pm 16.83}$ & $8.43 \, {\scriptscriptstyle \pm 4.85}$ & $96.15 \, {\scriptscriptstyle \pm 0.31}$ & $56.43 \, {\scriptscriptstyle \pm 14.25}$ & $4.901 \, {\scriptscriptstyle \pm 1.308}$ & $0.007 \, {\scriptscriptstyle \pm 0.000}$ \\
 & Gen-FVP & $66.04 \, {\scriptscriptstyle \pm 3.96}$ & $62.36 \, {\scriptscriptstyle \pm 2.35}$ & $69.82 \, {\scriptscriptstyle \pm 3.34}$ & $31.80 \, {\scriptscriptstyle \pm 20.59}$ & $4.65 \, {\scriptscriptstyle \pm 1.89}$ & $96.09 \, {\scriptscriptstyle \pm 2.09}$ & $63.60 \, {\scriptscriptstyle \pm 18.41}$ & $9.590 \, {\scriptscriptstyle \pm 2.807}$ & $0.027 \, {\scriptscriptstyle \pm 0.001}$ \\
\bottomrule
\end{tabular}}
\label{tab:r_compass}
\end{table*}

\end{document}